\documentclass[twoside]{article}

\usepackage[preprint]{aistats2027}

\usepackage{url}
\usepackage[hidelinks]{hyperref}

\makeatletter
\renewcommand{\@copyrightspace}{%
  \begin{figure}[b]
    \raggedright
    \rule{4.818pc}{0.4pt}\par\vspace{3pt}%
    {\small
      Authors are listed in alphabetical order.\\
      Correspondence: \texttt{matteo.papini@unimi.it}\par}
  \end{figure}%
}
\makeatother

\usepackage{algorithm, algorithmicx, algpseudocode}
\usepackage{xspace}
\usepackage{amsthm}
\usepackage{thm-restate}
\usepackage{xcolor}
\usepackage{natbib}
\usepackage{mathtools}
\usepackage{tabularray}
\usepackage{subcaption}
\usepackage{placeins}

\input{mymacros}
\newcommand{\osmd}{\textsc{OSMD}\xspace}
\newcommand{\algname}{\textsc{OSMA-W}\xspace}
\newcommand{\problem}{\textsc{FB+WF}\xspace}
\newtheorem{lemma}{Lemma}

\newtheorem*{remark}{Remark}

\usepackage{tabularray}
\definecolor{rowcolor}{gray}{0.9}

\begin{document}

%

%

\twocolumn[

\aistatstitle{
$m$-Set Adversarial Bandits with Winner Feedback
}

\aistatsauthor{ Nicolò Cesa-Bianchi \And Matteo Papini}

\aistatsaddress{ Università degli Studi di Milano \\ Politecnico di Milano \And Università degli Studi di Milano } ]

\begin{abstract}
We show upper and lower bounds on the regret of $m$-set adversarial bandits for different utilities (winner reward or sum of rewards) and feedback models (winner index, winner reward, sum of rewards, and their combinations). By comparing to standard bounds for combinatorial and MNL bandits, our results reveal how subtle changes in the setting can have a dramatic impact on the learning rates. Our main technical contributions are the information-theoretic lower bounds on the regret. Experiments on synthetic data confirm our theoretical analyses.
\end{abstract}

\section{INTRODUCTION}
Combinatorial bandits \citep{cesa2012combinatorial} are an established model of sequential decision-making with partial feedback and combinatorially structured decision sets. The special case of $m$-sets, when the decision set is all the $m$-sized subsets $S$ of a base set of $K > m$ actions, is especially interesting because of its many applications to recommendation and assortment selection. MNL bandits \citep{agrawal2019mnl} is an alternative setting specifically designed for this type of applications. Unlike combinatorial bandits, where the utility of a $m$-subset $S$ is the sum of the unknown rewards of the base actions in $S$, in MNL bandits rewards (profits) are known, and the expected utility of $S$ is the reward of a base action drawn from $S$. The distribution of this draw follows a Plackett-Luce model based on unknown attraction parameters associated with each base action, typically including a no-purchase option which has zero reward for the decision-maker. The feedback also differs in the two settings: in MNL bandits the player observes the index of the drawn base action (winner feedback), while combinatorial bandits distinguish between full bandit feedback (i.e., the utility of $S$) and semi-bandit feedback (i.e., the reward of each action in $S$).

In this work we introduce an intermediate setting bridging $m$-set adversarial bandits and MNL bandits. Like in combinatorial bandits, the utility of $S$ is the sum of the unknown and adversarially chosen rewards of the base actions in $S$. However, our feedback includes both full bandit feedback and winner feedback. Unlike typical MNL bandits, the attraction parameters are identified with the rewards, which are unknown. This feedback model is thus richer than full bandit feedback but poorer than semi-bandit feedback (we can simulate winner feedback using semi-bandit feedback). In the following, we refer to this model as \problem (Full Bandit plus Winner Feedback).

A natural application of this setting arises in slate-based recommendation systems, where a platform presents a user with a subset of products from a larger catalog. Each product has an unknown value for the user, and the quality of the displayed slate is modeled additively. The platform seeks to learn which collections of products yield high aggregate value over time. In many systems, however, feedback is only partially attributable to individual items. For example, after a recommendation session the platform may observe an aggregate measure of slate performance---such as total engagement, expected margin, or a session-level satisfaction score---while directly observing only which displayed product the user selected. Thus the learner observes the overall value generated by the slate and the identity of the item that ``won'' the user's attention, but not the individual rewards of all displayed products. This feedback model captures settings in which aggregate performance is measurable but fine-grained attribution is unavailable.

Our main result is the characterization, up to logarithmic factors, of the minimax regret of \problem, where the regret is measured against the best fixed $m$-set for the adversarial sequence of rewards. More precisely, we show that the minimax regret is of order $\widetilde{\Theta}\big(m\sqrt{KT}\big)$ where $T$ is the time horizon. Note that this rate falls exactly in between the corresponding minimax rates $\widetilde{\Theta}\big(m^{3/2}\sqrt{KT}\big)$ for full-bandit feedback and $\Theta\big(\sqrt{mKT}\big)$ for semi-bandit feedback. We also show that any algorithm that operates with winner feedback only has linear minimax regret, even if the feedback includes both index and reward of the winning arm.

A related model, introduced in a different context by \citet{AlaturL020}, assumes the feedback received by the learner is the index and reward of an arm drawn uniformly at random from $S$. We extend our techniques to prove a lower bound of order $m\sqrt{KT}$, matching (up to log factors) the upper bound proven in \citep{AlaturL020}. Note that this rate matches the one for \problem up to logarithmic factors. 


We also study a variant of our problem that is closer to MNL bandits, in which the utility of $S$ is not the sum of rewards but the expected reward of the winner. We show that, when the feedback is the sum of rewards, any learner suffers linear regret, even if the winner index is also revealed. Linear regret also occurs when the utility is the sum of rewards and the feedback is the index and reward of the winner. In both cases, the feedback distribution does not determine the expected utility of the actions. Our lower bounds exploit these facts by constructing pairs of reward distributions that induce the same feedback law for every action but have different optimal actions. This phenomenon is distinct from the scale invariance of the Plackett-Luce choice model (multiplying all rewards by a common positive factor leaves the winner distribution unchanged).

Finally, when the learner maximizes the winner reward, which is also provided as feedback (together with the winner index), we prove that the regret is $\Omega(\sqrt{(K/2m)^mT})$ and $\mathcal{O}(\sqrt{(eK/m)^mT})$. These rates are similar to adversarial MNL bandits with uncoupled attraction/utility parameters and a no-purchase option~\citep{han2021adversarial}, see the next section.

\textbf{Technical challenges.}
In our model, winner feedback discloses information about the individual rewards that is not captured by the total reward of the $m$-set. To design a good learning algorithm, we need to exploit this extra information and carefully combine it with the ordinary full-bandit feedback. The main idea is to simulate semi-bandit feedback and use known good algorithms for this setting~\citep{audibert2014regret}. Establishing an improved regret upper bound compared to full-bandit requires a careful handling of base-arm reward estimators.
However, the main challenge is to establish a matching regret lower bound. Although we can largely re-use constructions developed for the full-bandit feedback~\citep{audibert2014regret,CohenHK17,ito2019improved}, the standard information-theoretic arguments are broken by the extra information leaked by the winner's index. This requires both an ad-hoc reward construction and a technically involved KL divergence upper bound. Also for the alternative models we study, the main challenges pertain to the regret lower bounds.

\begin{table*}[t]
  \centering
  \begin{tblr}{Q[c,m]Q[c,m]Q[c,m]Q[c,m]Q[c,m]}
    \bfseries Utility &\bfseries Feedback &\bfseries Name &\bfseries Regret &\bfseries Reference \\ \hline
    \SetCell[r=2]{} {Sum of \\ rewards} &\SetCell[r=2]{} {Sum of \\ rewards} & \SetCell[r=2]{} Full-bandit& $\mathcal{O}\big(m^{3/2}\sqrt{KT\ln\frac{K}{m}}\big)$ & \cite{bubeck2012towards}\\
    & & & $\Omega\big(m^{3/2}\sqrt{KT}\big)$ & \cite{ito2019improved}\\
    \SetRow{bg=rowcolor}
    {Sum of \\ rewards} & {Reward of \\ each arm in $S$} & Semi-bandit & $\Theta\big(\sqrt{mKT}\big)$ & \cite{audibert2014regret} \\
    \SetCell[r=2]{} {Sum of \\ rewards} &\SetCell[r=2]{} {Winner index $+$ \\ sum of rewards} & \SetCell[r=2]{}\problem & $\mathcal{O}\big(m\sqrt{KT\ln\frac{K}{m}}\big)$ & \textbf{This paper} (Thm.~\ref{th:upper})\\
     & & & $\Omega\big(m\sqrt{KT}\big)$ & \textbf{This paper} (Thm.~\ref{th:lower2})\\
    \SetRow{bg=rowcolor}
    {Sum of \\ rewards} & {Winner index $+$ \\ winner reward} & & $\Omega(T)$ & \textbf{This paper} (Thm.~\ref{th:linearlb}) \\
        \SetCell[r=2]{} {Sum of \\ rewards} &\SetCell[r=2]{} {Reward of a \\ random arm in $S$} & \SetCell[r=2]{}{Random-arm or \\Multiplayer MAB} & $\mathcal{O}\big(m\sqrt{KT\ln K}\big)$ &\cite{AlaturL020} \\
     &  & & $\Omega\big(m\sqrt{KT}\big)$ & \textbf{This paper} (Thm.~\ref{th:lowermulti}) \\
    \SetRow{bg=rowcolor}
    {Winner \\ reward} & {Winner index $+$ \\ sum of rewards} & & $\Omega(T)$ &\textbf{This paper} (Thm.~\ref{th:profitlb})\\
    \SetCell[r=2]{} {Winner \\ reward} &\SetCell[r=2]{} {Winner index + \\ winner reward} & & $\mathcal{O}\big(\sqrt{\big(\frac{eK}{m}\big)^mT}\big)$ &\SetCell[r=2]{}\textbf{This paper} (Thm.~\ref{th:winwin}) \\
     & & & $\Omega\big(\sqrt{\big(\frac{K}{2m}\big)^mT}\big)$ & \\
  \end{tblr}
  \caption{Old and new results for $m$-set bandits. 
  }
  \label{tab:summary}
\end{table*}

\section{RELATED WORKS}
The minimax rates for the regret of $m$-set adversarial bandits were established by \citet{bubeck2012towards} (upper bound) and \cite{ito2019improved} (lower bound) for full bandit feedback, and by \citet{audibert2014regret} for semi-bandit feedback.

Recall that in MNL bandits arm attraction parameters (determining the winning arm) and arm rewards (determining the learner's utility) are uncoupled. Moreover, the results mentioned below here consider an extra no-purchase arm with a known constant attraction parameter.

The regret of MNL bandits was studied by \citet{agrawal2019mnl} with fixed and unknown attraction parameters and known rewards. When the utility is the reward of the winning arm and the feedback is the index of the winning arm, they establish an upper bound on the regret of order $\sqrt{KT}$ and a lower bound of order $\sqrt{KT/m}$. Note, however, that their upper bound holds for a large class of constraints (including the cardinality constraint considered in this work).

\citet{han2021adversarial} investigated MNL bandits with adversarial attraction parameters. They prove a regret lower bound of order $\sqrt{(K/m)^m T}$ for oblivious adversaries\footnote{The need for adaptive adversaries in \citep[Corollary~3]{han2021adversarial} is due to the statement being expressed as ``$\exists$ an adversary such that $\forall$ algorithms.'' If quantifiers are swapped, then it is easy to see that the same lower bound applies even to oblivious adversaries.} and winner index feedback. This result holds even when all rewards are equal to one. In this case, the binary feedback (i.e., purchase or no purchase) is a Bernoulli random variable which corresponds to the player's reward. This shows that the MNL model with (noisy) bandit feedback has a regret exponential in $m$ when the attraction parameters are allowed to be chosen adversarially. The exponential dependence in \cite[Corollary~3]{han2021adversarial} hinges on the no-purchase option, whose presence makes the expected reward of any base action a non-polynomial function of the base action attraction.

MNL bandits with fixed attraction parameters, without no-purchase option, and where the learner's utility is the sum of the rewards in $S$ were considered by \citet{saha2019combinatorial}. For this model, they prove an instance-dependent upper bound on the regret of order $\mathcal{O}(K\ln T)$ using winner index feedback only. Note, however, that here the player is allowed to play any set $S$ of size at most $m$ (as opposed to sets of size exactly $m$).

\citet{pmlr-v238-pasteris24a} studied a related adversarial model where the winning arm is any arm in $S$ achieving maximum reward (there are no attraction parameters), and the utility and feedback are just the reward of the winning arm. Similarly to the previous case, the player is allowed to play any set $S$ of size at most $m$. They prove an upper bound of order $\sqrt{mKT}$ on the $\alpha$-regret, where $\alpha = 1-\frac{1}{e}$. Assuming hardness of improving the logarithmic approximation factor of set cover, they also prove that no efficient algorithm can control $\alpha$-regret for any larger value of $\alpha$.

Finally, \citet{bhaskara2023online} consider the standard stochastic $K$-armed bandits augmented with probes. In their model closest to ours, the winning arm is any arm in $S$ with maximum realized reward, the utility is the reward of the winning arm, and the feedback is winner index and reward. They show a distribution-free upper bound on the regret of order $K^2\ln T$ when $m=2$. If the feedback includes the reward of all arms in $S$, the regret becomes $K^2$ (no dependence on $T$) for $m=3$.

\textbf{Notation.} 
For natural numbers $K,m$, let $[K]=\{1,2,\dots,K\}$ and ${[K]\choose m}$ be the set of all subsets of $[K]$ of size $m$. For a vector $x\in\Reals^d$, we denote its $i$-th component as $x(i)$.
For vectors $x,y\in\Reals^d$, let $\langle x,y\rangle$ denote their scalar product and $x\circ y$ their elementwise product. For any predicate $p$, let $\indi{p}$ be one if $p$ is true, zero otherwise.
For a set $\mathcal{A}\subseteq\{0,1\}^d$, let $\Conv(A)$ denote its convex hull and $\Delta(\mathcal{A})$ the set of probability distributions over $\mathcal{A}$. For an open convex set $\mathcal{D}\subseteq \Reals^d$, let $\overline{\mathcal{D}}$ denote its topological closure and $\partial \mathcal{D}=\overline{\mathcal{D}}\setminus\mathcal{D}$ its boundary. For a Legendre function $F:\overline{D}\to\Reals$, let $D_F:\overline{\mathcal{D}}\times\mathcal{D}\to\Reals$ denote the associated Bregman divergence, $D_F(x,y)=F(x)-F(y)-\langle x-y, \nabla F(y)\rangle$.

\section{PROBLEM STATEMENT}
The $m$-set adversarial bandit problem is a combinatorial bandit with $K$ base arms and decision set corresponding to all $m$-sized subsets of $[K]$. We consider an oblivious adversary, secretly assigning a reward $r_t(i)\in[0,1]$ for each round $t \in [T]$ and for each base arm $i\in[K]$, where $T$ is the given time horizon. At each round $t=1,\dots,T$, the learner plays a (possibly random) subset $S_t\subseteq[K]$ of cardinality $|S_t|=m<K$ and receives the total reward $r_t(S_t)$. For any $S\subseteq[K]$, the total reward is defined as:
\begin{equation}
    r_t(S) = \sum_{i\in S}r_t(i).
\end{equation}
In our \problem\ setting, the learner observes the full bandit feedback $r_t(S_t)$ and also a ``winner'' base arm $I_t\sim q_t(\cdot|S_t)$ sampled from the played subset according to the  \emph{Plackett-Luce choice model} with linear attraction parameters \citep{Luce1979IndividualCB,Plackett1975TheAO}:
\begin{equation}\label{eq:choice}
    q_t(i|S) = \begin{cases}
        \frac{r_t(i)}{r_t(S)} &\text{if } r_t(S)>0\\
        \frac{1}{|S|} &\text{otherwise},
    \end{cases}
\end{equation}
defined for any $S\subseteq[K]$ and all $i\in S$. Neither the choice probabilities nor the reward of the winner base arm are revealed to the learner. The goal of the learner is to minimize the expected total regret with respect to the best subset in hindsight, defined as:
\begin{equation}
    R_T = \max_{\substack{S\subseteq[K]\\ |S|=m}}\sum_{t=1}^T r_t(S) - \EE{\sum_{t=1}^T r_t(S_t)},
\end{equation}
where the expectation is with respect to the randomness of the learner's decisions and of the feedback.

\begin{remark}
    The original formulation of adversarial combinatorial bandits~\citep{audibert2014regret}, both with full-bandit and semi-bandit feedback, allows a constrained action set $\mathcal{A}\subset {[K]\choose m}$, while we only work with the unconstrained action space of all $m$-subsets.~\citet{ito2019improved} refer to these as \emph{multitask bandits}.
\end{remark}

\section{REGRET IN THE \problem MODEL}
In this section, we present a learning algorithm for the \problem\ model, and prove a worst-case regret upper bound for it. We derive a lower bound on the regret of the \problem problem that matches our upper bound up to logarithmic factors.

\subsection{Algorithm}
The algorithm is an adaptation of Online Stochastic Mirror Descent (\osmd) as proposed by \cite{audibert2014regret} for combinatorial bandits with semi-bandit feedback. We maximize instead of minimizing, add a forced-exploration term, and replace the base-arm rewards with unbiased estimates obtained by combining the full-bandit feedback $r_t(S_t)$ and the winner feedback $I_t$.

Let $\mathcal{A}=\{a\in\{0,1\}^K\mid \|a\|_1=m\}$.
An incidence vector $a\in\mathcal{A}$ can be identified with a subset $S_a\subseteq[K]$ of size $m$ containing exactly those indexes $i$ such that $a(i)=1$. 
The rewards of the $K$ arms at time $t$ form a reward vector $r_t$ with elements $r_t(i)$ for $i\in[K]$, and the total reward of $S_t=S_{a_t}$ can be written as $r_t(S_t) = \langle a_t, r_t\rangle$. 

For $x\in (0,+\infty)^K$, let $F(x)=\sum_{i=1}^Kx_i(\log x_i - 1)$ denote the (normalized) negative entropy, our distance-generating function for the mirror ascent scheme.
Our method is called Online Stochastic Mirror Ascent with Winner feedback (\algname), outlined in Algorithm \ref{alg:osma}. Let $p_t$ be the distribution from which the algorithm draws incidence vectors (Line \ref{ll:pt}). With little abuse of notation, we write $p_t(a)$ or $p_t(S) = p_t(S_{a})$ interchangeably given the one-to-one correspondence between subsets of $[K]$ and incidence vectors. 
Reward estimates are computed as follows for all $i\in[K]$:

\begin{equation}\label{eq:estim}
    \widehat{r}_t(i) = \frac{\indi{i\in S_t}\indi{i=I_t}r_t(S_t)}{\pi_t(i)},   
\end{equation}
where $\pi_t(i)=\sum_{S\subseteq [K]}\indi{i\in S}p_t(S)$ is the marginal of $p_t$ over $i \in [K]$.

\begin{algorithm}
    \begin{algorithmic}[1]
        \Statex \textbf{Input: learning rate $\eta>0$, $\epsilon\in (0,1)$} 
        \State $x_1 \in \arg\min_{x\in\Conv(\mathcal{A})}F(x)$ \label{ll:init}
        \For{$t=1,\dots,T$}
        \State Find $\widetilde{p}_t\in\Delta(\mathcal{A})$ s.t. $x_t=\EEp{a\sim \widetilde{p}_t}{a}$ \label{ll:pt_tilde}
        \State $p_t=(1-\epsilon)\widetilde{p}_t + \frac{\epsilon}{|\mathcal{A}|}$ \label{ll:pt}
        \State Sample $a_t\sim p_t$ and play $S_t=S_{a_t}$ \label{ll:sample}
        \State Observe $r_t(S_t)$ and $I_t$
        \State Compute reward vector estimate $\widehat{r}_t$ as in Eq. \eqref{eq:estim}
        \State Find $\widetilde{x}_{t+1}$ s.t. $\nabla F(\widetilde{x}_{t+1}) = \nabla F(x_t)+\eta \widehat{r}_t$ \label{ll:mirror}
        \State $x_{t+1} \in \arg\min_{x\in\Conv(\mathcal{A})} D_{F}(x,\widetilde{x}_{t+1})$ \label{ll:project}
        \EndFor
    \end{algorithmic}
    \caption{\algname}\label{alg:osma}
\end{algorithm}

\subsection{Regret Upper Bound}

First, we establish some important properties of the reward estimator (Eq.~\ref{eq:estim}). Let $\mathcal{F}_t$ be the $\sigma$-algebra generated by $S_1,I_1,\dots,S_{t},I_t$, and let $\EEt{\cdot}$ be short for $\EE{\cdot\mid\mathcal{F}_{t-1}}$.

\begin{lemma}\label{lem:estim}
    The reward estimates from Equation \eqref{eq:estim} are non-negative, upper bounded by $\widehat{r}_t(i)\le K/\epsilon$ for all $i\in[K]$, conditionally unbiased (i.e., $\EEt{\widehat{r}_t}=r_t$), and their conditional variance satisfies:
    \begin{equation*}
        \EEt{\sum_{i=1}^K \pi_t(i)\widehat{r}_t(i)^2} \le mK,
    \end{equation*}
    where $p_t$ is the sampling probability defined in Algorithm \ref{alg:osma} (Line \ref{ll:pt}).
\end{lemma}
\begin{proof}
    The estimates are clearly non-negative since $r_t(i)\in[0,1]$ for all $i\in[K]$. 
    The upper bound comes from the fact that $r_t(S_t)\le m$ and:
    \begin{equation*}
        \pi_t(i) = \!\!\!\!\!\sum_{a\in\mathcal{A}:a(i)=1}\!\!\! p_t(a) 
        \ge \!\!\!\!\!\sum_{a\in\mathcal{A}:a(i)=1}\!\! \frac{\epsilon}{|\mathcal{A}|} 
        = \epsilon\frac{{K-1 \choose m-1}}{{K \choose m}} 
        = \frac{\epsilon m}{K}.
    \end{equation*}
    The conditional expectation is:
    \begin{equation*}
    \begin{aligned}
        \EEt{\widehat{r}_t(i)} &= 
        \EEt{\EEt{\widehat{r}_t(i)\mid S_t}} \\
        &=\frac{1}{\pi_t(i)}\EEt{\indi{i\in S_t}r_t(S_t)\EEt{\indi{i=I_t}\mid S_t}} \\
        &=\frac{1}{\pi_t(i)}\EEt{\indi{i\in S_t}r_t(S_t)q_t(i|S_t)} \\
        &=\frac{r_t(i)}{\pi_t(i)}\EEt{\indi{i\in S_t}} = r_t(i),
        \end{aligned}
    \end{equation*}
    by the law of total expectation, the definition of $q_t$ (Eq. \ref{eq:choice}), and the fact that $\pi_t(i)=\Pr(i\in S_t\mid \mathcal{F}_{t-1})$. Note that $q_t(i|S_t)r_t(S_t)=r_t(i)$ both when $r_t(S_t)$ is and is not zero. Similarly, for the conditional second moment:
    \begin{equation*}
    \begin{aligned}
        &\EEt{\sum_{i=1}^K \pi_t(i)\widehat{r}_t(i)^2} \\&\quad = 
        \EEt{\sum_{i=1}^K\frac{\indi{i\in S_t}\indi{i=I_t}r_t(S_t)^2}{\pi_t(i)}} \\
        &\quad=   \sum_{i=1}^K\frac{1}{\pi_t(i)}\EEt{\indi{i\in S_t}r_t(S_t)^2\EEt{\indi{i=I_t}\mid S_t}} \\
        &\quad= \sum_{i=1}^K\frac{r_t(i)}{\pi_t(i)}\EEt{\indi{i\in S_t}r_t(S_t)} \\
        &\quad\le  m\sum_{i=1}^K\frac{r_t(i)}{\pi_t(i)}\EEt{\indi{i\in S_t}} = m\sum_{i=1}^K r_t(i) \le mK,
    \end{aligned}
    \end{equation*}
    because $\EEt{\indi{i=I_t}\mid S_t} = q_t(i|S) = \frac{r_t(i)}{r_t(S)}$.
\end{proof}

Compare this with the standard semi-bandit feedback: having access to the base arm rewards $r_t(i)$ of all $i\in S_t$, we could compute estimates simply as $\widetilde{r}_t(i)=\indi{i\in S_t}r_t(i)/\pi_t(i)$ for all $i\in [K]$, whose variance is just (cf. \citealp{audibert2014regret}):
\begin{equation}
    \EEt{\sum_{i=1}^K \pi_t(i)\widetilde{r}_t(i)^2} = \sum_{i=1}^K r_t(i) \le K.
\end{equation}
The extra $m$ factor in the variance of $\widehat{r}_t$ is due to the stochasticity of the winner feedback $I_t$, which is needed to reconstruct the base arm rewards.

\begin{restatable}{theorem}{upper}\label{th:upper}
    Let $\eta=\min\left\{\sqrt{\frac{\log\frac{K}{m}}{2KT}}, \frac{3-e}{K}\right\}$ and $\epsilon=\eta K$.
    The expected total regret of \algname (Algorithm \ref{alg:osma}) with negative entropy over $T$ rounds is upper bounded as
    \begin{equation*}
        R_T \le 2m\sqrt{2KT\log\frac{K}{m}},
    \end{equation*}
    for the \problem problem with $K$ base arms and $m$-sized subsets.
\end{restatable}
\begin{proof}[Proof Sketch]
    We adapt the analysis of \osmd with negative entropy for semi-bandit feedback (\citealp{audibert2014regret}), maximizing instead of minimizing and taking into account the extra factor $m$ in the variance of the base-arm reward proxies, which produce an extra $\sqrt{m}$ factor in the regret. Forced exploration is needed to upper-bound the reward estimates and, in turn, the exponential terms in the Bregman divergence. Setting its rate to $\epsilon=\eta K$ controls both these terms and the exploration cost, and optimizing the learning rate $\eta$ yields the claimed result. The full proof is deferred to Appendix~\ref{app:upper}.
\end{proof}

\begin{remark}
    If we worked with losses $\ell_t(i)=1-r_t(i)$ instead of rewards, the unbiased loss estimates would be:
    \begin{equation*}
        \widehat{\ell}_t(i) = 1-\frac{\indi{i\in S_t}\indi{i=I_t}(m-\ell_t(S_t))}{\pi_t(i)}.
    \end{equation*}
    These would be naturally bounded, making forced exploration unnecessary, but would lose the non-negativity required by the \osmd analysis. 
    This is also what prevents us from using the INF regularizer to remove the logarithmic factors as in~\citep{audibert2014regret}.
    To overcome this problem we may consider ``loser feedback'' 
 ($q_t$ proportional to losses instead of rewards), but that would be a conceptually very different interaction model.
\end{remark}

Compare this with state-of-the-art algorithms for closely related settings: for adversarial combinatorial bandits with full bandit feedback, \textsc{EXP-2} with John's exploration achieves $\widetilde{O}(m^{3/2}\sqrt{KT})$; for semi-bandit feedback, \osmd with an INF-like regularizer achieves $\mathcal{O}(\sqrt{mKT})$. In terms of dependence from subset size $m$, our result lies in between, suggesting that the extra winner feedback $I_t$ makes the full bandit problem easier, but not as easy as in the presence of semi-bandit feedback. In light of the $\Omega(m^{3/2}\sqrt{KT})$ lower bound by \cite{ito2019improved}, our Theorem \ref{th:upper} already establishes a separation between full bandit and \problem. In the next section, we show that a separation from semi-bandit also holds.

\subsection{Regret Lower Bound}

\begin{restatable}{theorem}{lowerbx}\label{th:lower2}
    Let $m\ge 2$, $K\ge 2m$, and $T>K/27$. For any strategy of the learner, there exists an instance of \problem such that
    \begin{equation*}
        R_T \ge \frac{m\sqrt{KT}}{1728}.
    \end{equation*}
\end{restatable}

\begin{proof}[Proof Sketch]
    We adapt the proof by \citet{ito2019improved} for multitask bandits with full-bandit feedback. The presence of the extra winner feedback requires a novel information-theoretic argument.
    The backbone of the proof is an averaging argument over all possible choices of $\lfloor m/2\rfloor$ good arms. Rewards are $r_t(i)=\frac{1}{4}\beta_t(i)+U_t$ where $\beta_t(i)\sim \ber(1/4+4\varepsilon\indi{\text{$i$ is a good arm}})$ and $U_t\sim\unif([1/4,1/2])$. 
    Each of these models is compared with an alternative model where one good arm is removed.
    The shared uniform noise is to make sure the winner's feedback doesn't leak too much information.
    The challenging step is showing that selecting the removed good arm incurs only $\mathcal{O}(\varepsilon^2/m)$ KL divergence per round. 
    The subtle extra information carried by the winner's index makes this significantly more involved compared to the full-bandit setting~\citep{audibert2014regret,CohenHK17,ito2019improved}.
    A complete proof is provided in Appendix~\ref{app:lower}.
\end{proof}
\begin{remark}
    The $\Omega(m\sqrt{KT})$ lower bound is of the same order of the one originally given by~\citet{audibert2014regret} for full-bandit feedback. It was not tight for their setting, as later showed by \citet{CohenHK17}. Neither result is comparable with ours, because their constructions require constrained action spaces. The correct comparison is with the $\Omega(m^{3/2}\sqrt{KT})$ lower bound by~\citet[][Theorem 1]{ito2019improved}.
\end{remark}
The regret lower bound from Theorem~\ref{th:lower2} matches the upper bound of our \algname algorithm (Thm.~\ref{th:upper}) up to logarithmic factors.
In light of the existing $\mathcal{O}(\sqrt{mKT})$ regret upper bounds for \osmd under semi-bandit feedback~\citep{audibert2014regret}, this shows that \problem is harder than semi-bandit.

\section{REGRET IN RELATED MODELS}\label{sec:other}

In this section, we establish regret upper and lower bounds for relevant variants of \problem with different utility, feedback, or both. The emerging picture is that \problem is a particularly favorable combination of utility and feedback.
Complete proofs can be found in Appendix~\ref{app:profit}. See Table~\ref{tab:summary} for a summary of the results.

We start by showing that observing the sum of rewards is necessary when it is the utility. Even observing the winner's index \emph{and} reward leads to linear regret.

\begin{restatable}{theorem}{linearlb}\label{th:linearlb}
    Consider a variant of \problem where the learner observes $(I_t,r_t(I_t))$ but does not observe $r_t(S_t)$. Let $m\ge 3$ and $K\ge 2m$. For any strategy of the learner, there exists an instance of this problem such that
    \begin{equation*}
        R_T \ge \frac{\min\{4(K-2-m), (m-2)(m-1)\}}{8(4(K-2)+(m-2)(m-1))} \: T.
    \end{equation*}
\end{restatable}
For example, if $K=6$ and $m=3$, $R_T\ge 0.01T$. However, note that the slope vanishes as $K/m^2\to\infty$.
\begin{proof}[Proof Sketch]
    We partition the base arms into special arms $\{1,2\}$ and $K-2$ common arms, $m-1$ of which are good. We construct two reward distributions $P,Q$ differing only by exchanging the roles of arms $1$ and $2$: the optimal action contains $1$ under $P$ and arm $2$ under $Q$. We show that the feedback $(I_t,r_t(I_t))$ has the same law under the two models for every played action, although every action has a strictly positive average suboptimality gap across the two models, hence at least one model yields linear regret.
\end{proof}

Next, we look at the related (idealized) Multiplayer MAB setting introduced by~\citet{AlaturL020}, which is essentially \problem where the choice model $q_t$ is replaced by a uniform distribution over the played subset $S_t$. \citet{AlaturL020} actually use losses $\ell_t(i)\in[0,1]$, but this is equivalent to our reward formulation by letting $r_t(i)=1-\ell_t(i)$. Under the uniform sampling, this does not affect the ordinal feedback. \citet{AlaturL020} provide an EXP3-like algorithm for the Idealized Multiplayer MAB with a $O(m\sqrt{KT\log K})$ regret upper bound. They do not provide a lower bound. We prove a matching (up to logarithmic factors) lower bound by adapting our proof of Theorem~\ref{th:lower2}. The main difference is the KL upper bound, which is actually much easier to show under uniform sampling. The following result applies to the original Idealized Multiplayer MAB by letting $\ell_t(i)=1-r_t(i)$.

\begin{restatable}{theorem}{lowermulti}\label{th:lowermulti}
    Consider a variant of \problem where the feedback is $Y_t=(J_t, r_t(J_t))$, where $J_t \sim \unif(S_t)$.
    Let $m\ge 2$, $K\ge 2m$, and $T>K/3$. For any strategy of the learner, there exists an instance of this problem such that
    \begin{equation*}
        R_T \ge \frac{m\sqrt{KT}}{576}.
    \end{equation*}
\end{restatable}

This confirms that Multiplayer MAB is harder than semi-bandit, and that it has the same rate of \problem even though their feedback models are not comparable (they cannot simulate each other). 
More interestingly, we can compare the aforementioned $\mathcal{O}(m\sqrt{KT})$ upper bound for Multiplayer MAB with Theorem~\ref{th:linearlb}. Together, they show that, when the sum of rewards $r_t(S_t)$ is not observed, observing the reward of a random arm is preferable (in a worst-case sense) to observing the reward of the winner.

We now turn to a utility function that is more typical of MNL bandits: the expected reward of the winner. We study this under the same choice model of \problem: Plackett-Luce with adversarial linear attraction parameters equal to the rewards and no no-purchase option. Unfortunately, under the same feedback model of \problem, the regret is linear. 

\begin{restatable}{theorem}{profitlb}\label{th:profitlb}
    Consider a variant of \problem in which, for every
    $S\subseteq[K]$ with $|S|=m$, the objective function is the expected
    reward of the winner:
    \begin{equation*}
        \rho_t(S)
        \coloneqq
        \sum_{i\in S}q_t(i\mid S)r_t(i)
        =
        \begin{cases}
            \displaystyle
            \frac{\sum_{i\in S}r_t(i)^2}{r_t(S)},
            & r_t(S)>0,\\[2mm]
            0,
            & r_t(S)=0,
        \end{cases}
    \end{equation*}
    and the regret is defined accordingly as
    \begin{equation*}
        R_T
        =
        \max_{\substack{S\subseteq[K]\\ |S|=m}}
        \sum_{t=1}^T\rho_t(S)
        -
        \EE{\sum_{t=1}^T\rho_t(S_t)},
    \end{equation*}
    but the feedback remains $Y_t=(r_t(S_t),I_t)$. Let $K>m\geq2$.
    For any strategy of the learner, there exists an oblivious reward
    sequence such that
        $R_T\geq\frac{T}{444}$.
\end{restatable}
\begin{proof}[Proof Sketch]
    We exploit a ``cannibalization effect'' specific to the winner's reward objective: adding a positive-reward arm can decrease the utility by diverting choice probability from an arm with larger reward. We focus on two \emph{core} arms $\{1,2\}$, completing the arm set with $K-2$ zero-reward dummy arms. We construct two reward distributions $P,Q$ under which the feedback $(r_t(S_t), I_t)$ has the same law for every action. Under $P$, adding arm $2$ improves the expected winner's reward, so an optimal action contains both core arms. Under $Q$, arm $2$ cannibalizes arm $1$. Every action has an average suboptimality gap of at least $1/444$ across these indistinguishable models.
\end{proof}

The outlook is only slightly better when the learner observes the winner's reward: the regret is sublinear in $T$, but exponential in $m$. This is true even if the winner's index is revealed.

\begin{restatable}{theorem}{winwin}\label{th:winwin}
    Consider a variant of \problem in which, for every
    $S\subseteq[K]$ with $|S|=m$, the objective function is the expected
    reward of the winner $\rho_t$, as defined in Theorem~\ref{th:profitlb}, but the learner observes the index and the reward of the winner, $Y_t=(I_t,r_t(I_t))$. Then, 
    a simple adaptation of INF~\citep{AudibertB10} achieves:
    \begin{equation*}
        R_T \le 8\sqrt{{K \choose m}T} = \mathcal{O}\left(\sqrt{\left(\frac{eK}{m}\right)^m T}\right).
    \end{equation*}
    Let $K>m\ge 2$ and $T\ge \frac{2^m{K\choose m}}{3}$. For any strategy of the learner, there exists an instance of this problem such that  
    \begin{align*}
    R_T &\ge \frac{1}{18}\sqrt{2^{-m}3{K\choose m}T} = \Omega\left(\sqrt{\left(\frac{K}{2m}\right)^mT}\right).
    \end{align*}
\end{restatable}
\begin{proof}[Proof Sketch]
    For the upper bound, we follow~\citet{han2021adversarial} and treat the ${K \choose m}$ actions as the arms of a MAB, applying INF to the observed winner rewards, ignoring the winner indexes. For the lower bound, we reproduce their averaging argument, replacing their non-constructive hard instance with an explicit parity perturbation of independent Bernoulli rewards. This perturbation preserves all proper marginals while increasing the expected winner's reward of a distinguished set. Moreover, by exchangeability, conditional on the reward, the winner's index is uniform under both models, providing no extra information. Bounding the resulting KL divergence and optimizing the perturbation size gives the stated lower bound.
\end{proof}

The latter is coherent with the results by~\citet{han2021adversarial} on adversarial combinatorial bandits with general link functions. However, their results are not directly comparable with ours since they consider uncoupled attraction and (known) utility parameters, and a no-purchase option. Note there is still a $2^{m/2}$ gap in the dependence on $m$.

\begin{figure*}[t]
    \centering
    \begin{subfigure}[t]{0.48\textwidth}
        \centering
        \includegraphics[width=\linewidth]{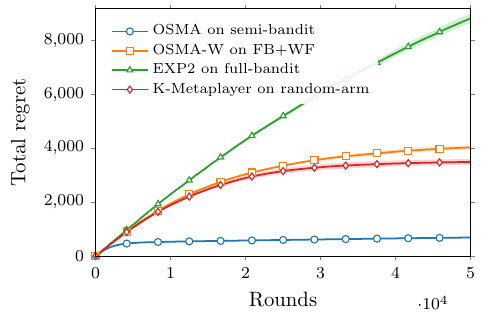}
        \caption{}
        \label{fig:correlated}
    \end{subfigure}
    \hfill
    \begin{subfigure}[t]{0.48\textwidth}
        \centering
        \includegraphics[width=\linewidth]{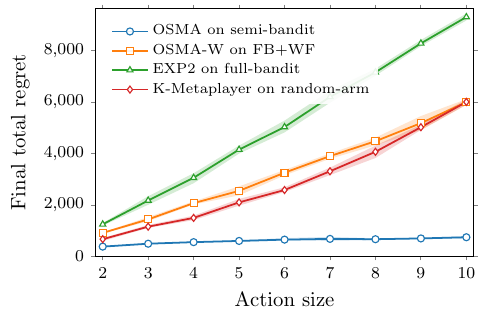}
        \caption{}
        \label{fig:action_size}
    \end{subfigure}
    \caption{Results of no-regret algorithms with their respective feedback on the stationary correlated reward instance. On the left (a), total regret per round, averaged over 10 seeds; on the right (b), final total regret after $20k$ rounds for increasing action size $m$, averaged over 5 seeds; 95\% bootstrap confidence intervals are shown.}
    \label{fig:combined}
\end{figure*}

\section{EXPERIMENTS}\label{sec:exp}
In this section, we test the algorithms discussed in the paper on synthetic $m$-set bandit instances.\footnote{Code at \url{https://github.com/T3p/konbini}.} We only consider feedback models for which regret sublinear in $T$ and polynomial in $K,m$ is possible.  Each algorithm is given access to the feedback for which it is designed. Namely, we evaluate OSMA on semi-bandit, \algname on \problem, $K$-Metaplayer on random-arm, and EXP2 with John's  exploration on full-bandit feedback. In all of these settings the objective function is the sum of rewards of the played arms, so we can directly compare their regret curves. All algorithm are run with the hyperparameters prescribed by the theory (see Appendix~\ref{app:exp} for details).

We first consider a stationary-reward setting with $K=10$, subsets of size $m=4$, and horizon $T=50000$. A set $G$ of $m$ good arms is randomly chosen and held fixed for the whole learning process. 
Reward generation is inspired by the hard instances from the proof of Theorem~\ref{th:lower2}.
For every round $t$, every base arm $i\in[K]$ is given reward $r_t(i)=\frac{1}{4}\beta_t(i) + U_{t}$, where $\beta_t(i)\sim\ber(1/4+ 4\varepsilon\indi{i\in G})$ and $U_{t}\sim \unif([1/4, 1/2])$ makes the base-arm rewards correlated within the same round. We set $\varepsilon=0.1$ in our experiments. Figure~\ref{fig:correlated} shows the total regret per round of the four algorithms. They play against the same reward sequence, generated once and for all. Results are averaged over $10$ random seeds governing the stochasticity of the algorithms and of the feedback model (winner, random-arm). In this and all other plots in this section, shaded areas correspond to 95\% pointwise bootstrap confidence intervals computed with 2000 bootstrap resamples. 
All regret curves are sublinear and their ordering is the one predicted by the theory even though $K$-Metaplayer behaves slightly better than \algname in the respective settings. In appendix~\ref{app:exp} we repeat the same experiment averaged over 100 independent reward-sequence instances to qualitatively analogous results.

Figure~\ref{fig:action_size} shows the final total regret after $20000$ rounds for increasing values of the action size $m$ on the same type of correlated stationary rewards. For all experiments, $K=20$, the number of good arms is $10$, and the reward sequence is held fixed. The range of values of $m$ is too small (due to computational constraints, the main bottleneck being EXP2) to claim anything about the order of the polynomial dependence of the regret on $m$. However, not only the ordering of the different feedback models is coherent with the theory, but, as predicted, the gaps get wider as $m$ increases.

Figure~\ref{fig:corrupted} shows the total regret per round on a variant of the stationary-reward setting where 20\% of rounds are corrupted. A set $D$ of $m$ decoy arms is sampled at the beginning of the process. For each $t$, with probability $0.2$, the rewards of arms in $D$ are set to $1$. The results are similar to the stochastic case.

\begin{figure}[t]
    \centering
    \includegraphics[width=\linewidth]{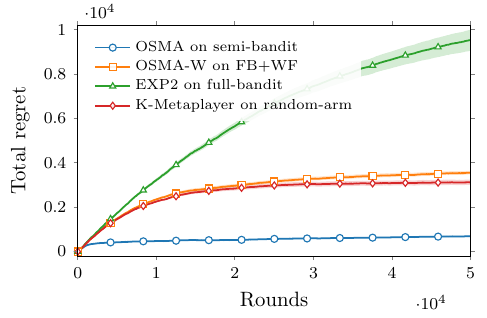}
    \caption{Results on corrupted rewards. Total regret per round, averaged over 10 seeds, with 95\% bootstrap confidence intervals.}
    \label{fig:corrupted}
\end{figure}

In Appendix~\ref{app:exp} we provide some further results on different kinds of reward instances: uncorrelated uniform and Bernoulli rewards, and time-varying (block-wise stationary) rewards. The results are again similar.

\section{CONCLUSION}
\algname admits nontrivial potential improvements. First, running \osmd with regularizers different from the negative entropy may allow removing the logarithmic factors from the regret, although the INF regularizer proposed by~\citet{audibert2014regret} does not apply directly, as already discussed. Second, our algorithm currently only works with the complete ${[K]\choose m}$ action space. The main challenge to handle constrained action spaces $\mathcal{A}\subset {[K]\choose m}$ would be to adapt the forced exploration mechanism. Finally, an algorithm that is computationally efficient also with constrained action spaces could be developed using ideas from \citep{NeuB16}, although it is not clear how to make that work with the reward proxies.

Related models worth studying are: \emph{loser feedback}, where choice probability is proportional to losses instead of rewards (we suspect this is an easier setting); non-linear attractions, as in the exponential Plackett-Luce model; the impact of a no-purchase option; the effects of decoupling attractions from revenues.




\FloatBarrier
\clearpage

\subsubsection*{Acknowledgements}
We thank Lukas Zierahn for contributing to a preliminary version of this work. NCB and MP acknowledge the financial support from the EU Horizon CL4-2022-HUMAN-02 research and innovation action under grant agreement 101120237, project ELIAS (European Lighthouse of AI for Sustainability).

\bibliographystyle{apalike}
\bibliography{references}

\clearpage
\appendix
\thispagestyle{empty}

\onecolumn

\section{Proof of Theorem \ref{th:upper}}\label{app:upper}

Our proof of Theorem \ref{th:upper} is an adaptation of Theorem 3 from \cite{audibert2014regret}, which upper bounds the regret of \osmd with negative entropy for adversarial combinatorial bandits with semi-bandit feedback. There are two differences in our analysis: we maximize instead of minimizing, since we work entirely with rewards; and the rewards of the base arms, even those belonging to the played subset $S_t$, are replaced by unbiased estimates (Eq. \ref{eq:estim}). 

First, we recall some fundamental convex optimization concepts. Given an open convex set $\mathcal{D}\subseteq \Reals^K$, a continuous function $F:\overline{D}\to\Reals$ is \emph{Legendre} if it is strictly convex and continuously differentiable on $\mathcal{D}$, and $\lim_{x\to\partial\mathcal{D}}\|\nabla F(x)\|=+\infty$. The Bregman divergence $D_F:\overline{\mathcal{D}}\times\mathcal{D}\to\Reals$ generated by $F$ is defined as
\begin{equation}
    D_F(x,y)=F(x)-F(y)-\langle x-y, \nabla F(y)\rangle,
\end{equation}
and its Legendre-Fenchel transform $F^*:\mathcal{D}^*\to\Reals$ is defined as
\begin{equation}
    F^*(u) = \sup_{x\in\overline{\mathcal{D}}}(\langle x, u\rangle - F(x)),
\end{equation}
where $D^*=\{\nabla F(x)\mid x\in \mathcal{D}\}$ is the dual space. If $F$ is Legendre, $F^*$ is Legendre, $F^{**}=F$, $\nabla F^*=(\nabla F)^{-1}$ is a one-to-one mapping, and $D_{F}(x,y)=D_{F^*}(\nabla F(y), \nabla F(x))$.

Next, we restate a fundamental property of \osmd (\cite{audibert2014regret}, Theorem 2), adapted to our reward-maximization framework with forced exploration. The proof is rather standard, but we include it for completeness.

\begin{lemma}\label{lem:osma}
    Consider a generalization of Algorithm \ref{alg:osma} where $F$ is any Legendre function with domain $\mathcal{D}$ such that $\Conv(\mathcal{A})\subseteq \overline{\mathcal{D}}$, and $\widehat{r}_t$ is any unbiased reward estimator such that, for all $x\in \Conv(\mathcal{A})\cap \mathcal{D}$, $\nabla F(x)+\eta\widehat{r}_t \in \mathcal{D}^*$ for all $t\in[T]$. Then its expected regret is upper bounded as
    \begin{equation*}
        R_T \le \frac{\sup_{a\in\mathcal{A}}F(a)-F(x_1)}{\eta} + \frac{1}{\eta}\sum_{t=1}^T\EE{D_{F^*}\left(\nabla F(x_t)+\eta \widehat{r}_t, \nabla F(x_t)\right)} + \epsilon mT.
    \end{equation*}
\end{lemma}
\begin{proof}
    Consider a version of the algorithm without forced exploration ($\epsilon=0$). Namely, let $\widetilde{a}_t \sim \widetilde{p}_t$ and $\widetilde{S}_t = S_{\widetilde{a}_t}$. Note that $\EEt{\widetilde{a}_t}=x_t$ by Line \ref{ll:pt_tilde} of Algorithm \ref{alg:osma}.
    The regret can be decomposed as
    \begin{align}
        R_T = \underbrace{\max_{S\subseteq[K], |S|=m}\sum_{t=1}^T r_t(S) - \EE{\sum_{t=1}^T r_t(\widetilde{S}_t)}}_{(A)} + \underbrace{\EE{\sum_{t=1}^T \left(r_t(\widetilde{S}_t)-r_t(S_t)\right)}}_{(B)}. \label{eq:decomp}
    \end{align}

    For the first term we follow the proof of Theorem 2 from \cite{audibert2014regret}. Fix any $S\subseteq [K]$ such that $|S|=m$, and let $a\in\mathcal{A}$ such that $S=S_a$. Then,
    \begin{align}
        \sum_{t=1}^T r_t(S) - \EE{\sum_{t=1}^T r_t(\widetilde{S}_t)}
        &= \EE{\sum_{t=1}^T\langle a - \widetilde{a}_t, r_t \rangle} \\
        &= \EE{\sum_{t=1}^T\langle a - \EEt{\widetilde{a}_t}, r_t \rangle} \\ 
        &= \EE{\sum_{t=1}^T\langle a - x_t, r_t \rangle} \\
        &= \EE{\sum_{t=1}^T\langle a - x_t, \EEt{\widehat{r}_t} \rangle} \label{eq:rt} \\
        &= \EE{\sum_{t=1}^T\langle a - x_t, \widehat{r}_t \rangle} \label{eq:measure} \\
        &= \frac{1}{\eta}\EE{\sum_{t=1}^T\langle a - x_t, \nabla F(\widetilde{x}_{t+1}) - \nabla F(x_t) \rangle}, \label{eq:A1}
    \end{align}
    where Equation \eqref{eq:rt} is by Lemma \ref{lem:estim}, Equation \eqref{eq:measure} relies on the fact that $x_t$ is $\mathcal{F}_{t-1}$-mesurable, and the last equality is by Line \ref{ll:mirror} of Algorithm \ref{alg:osma}.
    By the three-point identity of the Bregman divergence \citep{ChenT93}:
    \begin{equation}
        D_F(a, \widetilde{x}_{t+1}) = D_F(a, x_t) + D_F(x_t,\widetilde{x}_{t+1}) - \langle a - x_t, \nabla F(\widetilde{x}_{t+1}) - \nabla F(x_t)\rangle.
    \end{equation}
    Also, by the generalized Pythagorean inequality \citep{bauschke1997legendre} and the fact that $x_{t+1}$ is the Bregman projection of $\widetilde{x}_{t+1}$ onto a convex set containing $a$:
    \begin{equation}
        D_F(a,\widetilde{x}_{t+1}) \ge D_F(a,x_{t+1}) + D_F(x_{t+1}, \widetilde{x}_{t+1}).
    \end{equation}
    Hence,
    \begin{align}
        \langle a - x_t, \nabla F(\widetilde{x}_{t+1}) - \nabla F(x_t)\rangle
        &= D_F(a, x_t) - D_F(a, \widetilde{x}_{t+1}) + D_F(x_t, \widetilde{x}_{t+1}) \\
        &\le D_F(a, x_t) - D_F(a, x_{t+1}) - D_F(x_{t+1},  \widetilde{x}_{t+1}) + D_F(x_t, \widetilde{x}_{t+1}).
    \end{align}
    Summing over $t$, the first difference telescopes:
    \begin{align}
        \sum_{t=1}^T\langle a - x_t, \nabla F(\widetilde{x}_{t+1}) - \nabla F(x_t)\rangle 
        &\le D_F(a,x_1) - D_F(a, x_{T+1}) + \sum_{t=1}^T\left(D_F(x_t, \widetilde{x}_{t+1})- D_F(x_{t+1},  \widetilde{x}_{t+1})\right) \\
        &\le D_F(a,x_1)+ \sum_{t=1}^T D_F(x_t, \widetilde{x}_{t+1}) \\
        &=D_F(a,x_1) + \sum_{t=1}^TD_{F^*}(\nabla F(\widetilde{x}_{t+1}), \nabla F(x_t)) \\
        &=D_F(a,x_1) + \sum_{t=1}^TD_{F^*}(\nabla F(x_t)+\eta\widehat{r}_t, \nabla F(x_t)) \\
        &\le F(a) - F(x_1) + \sum_{t=1}^TD_{F^*}(\nabla F(x_t)+\eta\widehat{r}_t, \nabla F(x_t)).\label{eq:A2}
    \end{align}
    The last inequality is from Line \ref{ll:init} of Algorithm \ref{alg:osma} and first-order optimality conditions for $x_1$:
    \begin{align}
        \langle a-x_1, \nabla F(x_1)\rangle \ge 0 \implies D_F(a,x_1) \le F(a) - F(x_1).
    \end{align}
    Plugging Equation \eqref{eq:A2} into Equation \eqref{eq:A1}, since the latter holds for any feasible $S$, we obtain the following upper bound on the first term of Equation \ref{eq:decomp}:
    \begin{equation}
        (A) \le  \frac{\sup_{a\in\mathcal{A}}F(a)-F(x_1)}{\eta} + \frac{1}{\eta}\sum_{t=1}^T\EE{D_{F^*}\left(\nabla F(x_t)+\eta \widehat{r}_t, \nabla F(x_t)\right)}.
    \end{equation}
    
    It remains to bound the second term of Equation \eqref{eq:decomp}. Let
      \begin{align}
        z_t &\coloneqq \EEp{a\sim p_t}{a} \\
        &= (1-\epsilon)\EEp{a\sim \widetilde{p}_t}{a} + \epsilon\frac{1}{|\mathcal{A}|}\sum_{a\in\mathcal{A}}a\\
        &= (1-\epsilon)x_t + \epsilon\frac{1}{{K \choose m}}{K-1 \choose m-1}\mathbf{1} \\
        &= (1-\epsilon)x_t +\epsilon \frac{m}{K}\mathbf{1},
    \end{align}
    and note that $\EEt{a_t}=z_t$ by Line \ref{ll:sample} of Algorithm \ref{alg:osma}.
    Then,
    \begin{align}
        (B) &= \EE{\sum_{t=1}^T \left\langle \widetilde{a}_t - a_t, r_t\right\rangle} \\
        &= \EE{\sum_{t=1}^T \left\langle \EEt{\widetilde{a}_t - a_t}, r_t\right\rangle} \\
        &= \EE{\sum_{t=1}^T \left\langle x_t-z_t, r_t\right\rangle} \\
        &= \epsilon\EE{\sum_{t=1}^T \left\langle x_t-\frac{m}{K}\mathbf{1}, r_t\right\rangle} \\
        &\le \epsilon \sum_{t=1}^T\langle x_t, r_t\rangle \\
        &\le \epsilon mT
    \end{align}
    concluding the proof.
\end{proof}

We can now prove Theorem \ref{th:upper} by adapting Theorem 3 from \cite{audibert2014regret} to our setting, using Lemma \ref{lem:osma} and the properties of our reward estimator established in Lemma \ref{lem:estim}.

\upper*
\begin{proof}
    It is well known, and easily verified, that the negative entropy is Legendre with domain $\mathcal{D}=(0,+\infty)^K$, so $\Conv(\mathcal{A})=\{a\in[0,1]^K\mid \|a\|_1= m\}$ is contained in the domain's closure $\overline{\mathcal{D}}$. Since the gradient is $\nabla F(x)=\log x$, the dual space is $\mathcal{D}^*=\Reals^K$. This means that the condition on the mirror ascent update from Lemma \ref{lem:osma} is automatically verified. Morover, the reward estimator $\widehat{r}_t$ from Equation \eqref{eq:estim} is unbiased by Lemma \ref{lem:estim}, hence the upper bound from Lemma \ref{lem:osma} holds.

    First, consider the first term from Lemma \ref{lem:osma}. For the negative entropy, Line \ref{ll:init} of Algorithm \ref{alg:osma} yields $x_1 = \frac{m}{K}\mathbf{1}$ and $F(x_1)=m\log\frac{m}{K}-m$, and $F(a)=-m$ for all $a\in\mathcal{A}$. Hence:
    \begin{equation}
        \frac{\sup_{a\in\mathcal{A}}F(a)-F(x_1)}{\eta} = \frac{m}{\eta}\log\frac{K}{m}.\label{eq:first}
    \end{equation}
    
    Let us now consider the second term from Lemma \ref{lem:osma}.
    Since $F$ is Legendre, $\nabla F$ is invertible, so there exists $y_t\in\Reals^d$ such that $\nabla F(y_t)=\nabla F(x_t)+\eta \widehat{r}_t$, that is, $\log y_t = \log x_t + \eta \widehat{r}_t$. Moreover,
    \begin{align}
        D_{F^*}(\nabla F(y_t),\nabla F(x_t)) = D_F(x_t, y_t) &= \sum_{i=1}^K \left(x_t(i)(\log x_t(i) - \log y_t(i))-x_t(i)+y_t(i)\right) \\
        &= \sum_{i=1}^K \left(x_t(i)(\log x_t(i) - \log y_t(i))-x_t(i)+y_t(i)\right) \\
        &= \sum_{i=1}^K x_t(i)\left(-\eta\widehat{r}_t(i)-1+e^{\eta\widehat{r}_t(i)}\right).\label{eq:exponential}
    \end{align}
    Suppose $\eta K \le \epsilon$. Then, by Lemma \ref{lem:estim}, $\eta\widehat{r}_t(i) \le \frac{\eta K}{\epsilon} \le 1$. Using the fact that $e^x-x-1 \le (e-2)x^2$ for all $x\le 1$:
    \begin{equation}
        -\eta\widehat{r}_t(i)-1+e^{\eta\widehat{r}_t(i)} \le (e-2)\eta^2\widehat{r}_t(i)^2.\label{eq:trick}
    \end{equation}
    Plugging this into Equation \eqref{eq:exponential}:
    \begin{align}
        \frac{1}{\eta}\sum_{t=1}^T\EE{D_{F^*}\left(\nabla F(x_t)+\eta \widehat{r}_t, \nabla F(x_t)\right)} 
        &\le (e-2)\eta \sum_{t=1}^T\EE{\sum_{i=1}^K x_t(i)\widehat{r}_t(i)^2} \\
        &\le (e-2)\eta \EE{\sum_{t=1}^T\sum_{i=1}^K x_t(i)\EEt{\widehat{r}_t(i)^2}}.\label{eq:quadratic}
    \end{align}
    Recall, from the proof of Lemma \ref{lem:estim}, that:
    \begin{align}
        \pi_t(i) &= \sum_{a\in\mathcal{A}:a(i)=1}p_t(a) = (1-\epsilon)\sum_{a\in\mathcal{A}:a(i)=1}\widetilde{p}_t(a) + \frac{\epsilon m}{K} \\
        &\ge (1-\epsilon)\EEp{a\sim\widetilde{p}_t}{a(i)} \\
        &=(1-\epsilon)x_t(i),
    \end{align}
    so $x_t(i)\le\frac{\pi_t(i)}{1-\epsilon}$. Using this in Equation \eqref{eq:quadratic}:
    \begin{align}
         \frac{1}{\eta}\sum_{t=1}^T\EE{D_{F^*}\left(\nabla F(x_t)+\eta \widehat{r}_t, \nabla F(x_t)\right)} 
         &\le \frac{(e-2)}{1-\epsilon}\eta \EE{\sum_{t=1}^T\sum_{i=1}^K \pi_t(i)\EEt{\widehat{r}_t(i)^2}} \\
         &\le \frac{(e-2)\eta mKT}{1-\epsilon}.
    \end{align}
    
    By plugging the latter inequality and Equation \ref{eq:first} into the upper bound of Lemma \ref{lem:osma}, we obtain:
    \begin{equation}
        R_T \le \frac{m}{\eta}\log\frac{K}{m} + \frac{(e-2)\eta mKT}{1-\epsilon} + \epsilon m T.
    \end{equation}

    Let $\epsilon=\eta K$, the smallest value that satisfies the condition required by Equation \eqref{eq:trick}. For this choice of $\epsilon$, forcing $\eta \le \frac{3-e}{K}$ (which also ensures $\epsilon\le 3-e < 1$):
        \begin{equation}\label{eq:find_eta}
        R_T \le \frac{m}{\eta}\log\frac{K}{m} + 2\eta mKT.
    \end{equation}
    Solving for $\eta$, we obtain the optimal value $\eta^\star=\sqrt{\log(K/m)/(2KT)}$. When $\eta^\star \le (3-e)/K$, the proof is concluded by substituting the value of $\eta^\star$ in Equation \eqref{eq:find_eta}. If instead, $\eta > (3-e)/K$, then it must be:
    \begin{equation*}
        T < \frac{K\log\frac{K}{m}}{2(3-e)^2} < 7K\log\frac{K}{m}.
    \end{equation*}
    Since the regret at each round is at most $m$, $R_T\le mT$. However, this is smaller than the claimed bound for any $T\le 8K\log\frac{K}{m}$.
\end{proof}

\section{Proof of Theorem \ref{th:lower2}} \label{app:lower}

We first establish a useful information-theoretic fact. We are going to compare two probability models (denoted by capital letters $P$ and $Q$). For random elements, such as $Y$, whose law may differ under the two models, we are going to denote their law as $P_Y$ and $Q_Y$, respectively. We use $\EEp{P}{\cdot}$ and $\EEp{Q}{\cdot}$ to denote expectations under the two models.

\begin{lemma}\label{lem:kl}
    Let $0<a<b<1$, $p\in(0,1)$, and $\delta>0$ such that $p+\delta < 1$. Let $m\ge 2$ and, for each $i\in[m]$, let $q_i\in(0,1)$. 
    Consider two probability models $P$ and $Q$ and define random variables $U,B_1,\dots,B_m$ as follows:
    \begin{itemize}
        \item $B_1 \sim \ber(p+\delta)$ under $P$ and $B_1\sim \ber(p)$ under $Q$.
        \item For $i=2,\dots,m$, $B_i\sim \ber(q_i)$ under both $P$ and $Q$.
        \item $U\sim\unif([a,b])$ under both $P$ and $Q$;
        \item Under each model, $U,B_1,\dots,B_m$ are mutually independent.
    \end{itemize}

Let $0<\eta<m(b-a)$. For $i\in[m]$, let $X_i=U+\eta B_i$, $X=(X_1,\dots,X_m)$, and
\begin{equation*}
    Z = \sum_{i=1}^m X_i.
\end{equation*}
Let $I$ be a random index supported on $[m]$ and distributed as follows conditionally on $X$:
\begin{equation*}
    \Pr(I=i\mid X) = \frac{X_i}{Z}.
\end{equation*}
Finally, let $Y=(Z,I)$. Then, $X_i\in[a,b+\eta]$ a.s., and
    \begin{equation*}
        \kl(Q_Y\|P_Y) \le \frac{2\eta\delta^2}{m\min\{p+\delta,1-p-\delta\}}\left(\frac{2}{b-a}+\frac{1}{a}\right).
    \end{equation*}
\end{lemma}

\begin{proof}
    That $X_i\in[a,b+\eta]$ a.s.\ is clear from its definition.
    Let $H=B_1$ and $G=(B_2,\dots,B_m)$. Note that $G$ has the same law under $P$ and $Q$. For $g\in\{0,1\}^{m-1}$ and $h\in\{0,1\}$, let $\nu_h^g$ denote the conditional law of $Y$ given $H=h$ and $G=g$. This conditional law is the same under $P$ and $Q$. For $h\in\{0,1\}$, let $\nu_h=\EE{\nu_h^G}$. Then:
    \begin{equation}
        P_Y = (1-p-\delta)\nu_0 + (p+\delta)\nu_1, \qquad Q_Y = (1-p)\nu_0 + p\nu_1.
    \end{equation}
    Let $\Omega=\Reals\times [m]$ and $\mathcal{F}$ be its Borel $\sigma$-algebra. By Lemma \ref{lem:mixtures} with $q=p, q'=p+\delta$:
    \begin{equation}\label{eq:kl}
        \kl(Q_Y\|P_Y) \le \frac{2\delta^2}{\min\{p+\delta,1-p-\delta\}}\tv(\nu_0,\nu_1).
    \end{equation}
    The remainder of the proof is concerned with upper-bounding the above total variation distance.
    
    For $g\in\{0,1\}^{m-1}$, let $s(g)=\sum_{i=1}^{m-1}g_i$. Conditional on $H=h$ and $G=g$,
    \begin{equation}\label{eq:conditionals}
        \begin{aligned}
            & X_1 = U + \eta h, \\
            &X_i = U + \eta g_{i-1} \qquad \text{for } i=2,\dots,m,\\
            &Z = mU + \eta(h+s(g)).
        \end{aligned}
    \end{equation}
    Let $\lambda_h^g$ be the law of $Z$ conditional on $H=h$ and $G=g$, which is uniform over the following interval: 
    \begin{equation}
        \mathcal{I}_h^g \coloneqq [ma+\eta(h+s(g)), mb+\eta(h+s(g))].
    \end{equation}
    For any $h\in\{0,1\},g\in\{0,1\}^{m-1}$, the length of $\mathcal{I}_h^g$ is equal to $L\coloneqq |\mathcal{I}_h^g| = m(b-a)$. Moreover, $\mathcal{I}_1^g=\mathcal{I}_0^g + \eta$ and, since $\eta <  m(b-a)$ by assumption, the two intervals overlap. In particular:
    \begin{equation}
        |\mathcal{I}_0^g\setminus\mathcal{I}_1^g| = |\mathcal{I}_1^g\setminus\mathcal{I}_0^g| = \eta, \qquad |\mathcal{I}_0^g\cap \mathcal{I}_1^g| = L-\eta.
    \end{equation}
    Since $\lambda_h^g$ is uniform:
    \begin{equation}\label{eq:measure2}
        \lambda_h^g(\mathcal{I}_0^g\setminus\mathcal{I}_1^g) = \lambda_h^g(\mathcal{I}_1^g\setminus\mathcal{I}_0^g) = \frac{\eta}{L}, 
        \qquad \lambda_h^g(\mathcal{I}_0^g\cap \mathcal{I}_1^g) = \frac{L-\eta}{L}.
    \end{equation}
    Fix a point $z\in \mathcal{I}_0^g\cap \mathcal{I}_1^g$. Let $u_h^g(z)$ be the value of $U$ such that, conditional on $H=h$ and $G=g$, $Z=z$. By Equation \eqref{eq:conditionals}:
    \begin{equation}
        u_h^g(z) = \frac{z-\eta(h+s(g))}{m}.
    \end{equation}
    Note that $u_0^g(z)-u_1^g(z)=\eta/m$. Define:
    \begin{equation}
        x_1(z;h,g)\coloneqq u_h^g(z)+\eta h,\qquad x_i(z;h,g)\coloneqq u_h^g(z) + \eta g_{i-1} \quad\text{for } i=2,\dots,m.   
    \end{equation}
    These are the values of $X_1,\dots,X_m$ such that, conditional on $H=h$ and $G=g$, $Z=z$. Let $\pi_h^g(\cdot|z)$ be the conditional distribution of $I$ given $H=h$, $G=g$ and $Z=z$, that is:
    \begin{equation}
        \pi_h^g(i|z) = \frac{x_i(z;h,g)}{z}.
    \end{equation}
    Then, recalling that $m\ge 2$,
    \begin{align}
        \tv(\pi_0^g(\cdot|z), \pi_1^g(\cdot|z)) 
        &= \frac{1}{2z}\sum_{i=1}^m \left|x_i(z;0,g) - x_i(z;1,g)\right| \\
        &= \frac{1}{2z}\left(\left|u_0^g(z)-u_1^g(z)-\eta\right| + \sum_{i=2}^m\left|u_0^g(z)-u_1^g(z)\right|\right) \\
        &=\frac{1}{2z}\left(\eta\left(1-\frac{1}{m}\right) + (m-1)\frac{\eta}{m}\right)\\
        &=\frac{\eta}{z}\left(1-\frac{1}{m}\right)\\
        &< \frac{\eta}{z} \\
        &\le \frac{\eta}{ma}, \label{eq:tv}
    \end{align}
    since $z\in\mathcal{I}_0^g\cap \mathcal{I}_1^g$ implies $z\ge ma$.

    For any measurable set $\mathcal{Y}\subseteq \Reals\times[m]$, define the sections:
    \begin{equation}
        \mathcal{Y}_z \coloneqq \{i\in[m]:(z,i)\in \mathcal{Y}\} \quad \text{for } z\in\Reals.
    \end{equation}
    Then, for $h\in\{0,1\}$,
    \begin{equation}
        \nu_h^g(\mathcal{Y}) = \int_{\mathcal{I}_h^g}\pi_h^g(\mathcal{Y}_z|z)\lambda_h^g(\de z).
    \end{equation}
    In particular:
    \begin{equation}
        \begin{aligned}
            &\nu_0^g(\mathcal{Y}) = \int_{\mathcal{I}_0^g\setminus \mathcal{I}_1^g}\pi_0^g(\mathcal{Y}_z|z)\lambda_0^g(\de z) + \int_{\mathcal{I}_0^g\cap \mathcal{I}_1^g}\pi_0^g(\mathcal{Y}_z|z)\lambda_0^g(\de z),\\
            &\nu_1^g(\mathcal{Y}) = \int_{\mathcal{I}_1^g\setminus \mathcal{I}_0^g}\pi_1^g(\mathcal{Y}_z|z)\lambda_1^g(\de z) + \int_{\mathcal{I}_0^g\cap \mathcal{I}_1^g}\pi_1^g(\mathcal{Y}_z|z)\lambda_1^g(\de z).
        \end{aligned}
    \end{equation}
    By the triangle inequality:
    \begin{align}
        \left|\nu_0^g(\mathcal{Y})-\nu_1^g(\mathcal{Y})\right|
        &\le \underbrace{\left|\int_{\mathcal{I}_0^g\setminus \mathcal{I}_1^g}\pi_0^g(\mathcal{Y}_z|z)\lambda_0^g(\de z)\right|}_{(A)}
        + \underbrace{\left| \int_{\mathcal{I}_1^g\setminus \mathcal{I}_0^g}\pi_1^g(\mathcal{Y}_z|z)\lambda_1^g(\de z) \right|}_{(B)}
        \\&\qquad+ \underbrace{\left|\int_{\mathcal{I}_0^g\cap \mathcal{I}_1^g}\pi_0^g(\mathcal{Y}_z|z)\lambda_0^g(\de z) -  \int_{\mathcal{I}_0^g\cap \mathcal{I}_1^g}\pi_1^g(\mathcal{Y}_z|z)\lambda_1^g(\de z)\right|}_{(C)}.
    \end{align}
    Since $0\le \pi_h^g(\mathcal{Y}_z|z)\le 1$, by Equation \eqref{eq:measure2}:
    \begin{equation}
        (A) \le \int_{\mathcal{I}_0^g\setminus \mathcal{I}_1^g}1\lambda_0^g(\de z) = \lambda_0^g(\mathcal{I}_0^g\setminus \mathcal{I}_1^g) = \frac{\eta}{L} = \frac{\eta}{m(b-a)}.
    \end{equation}
    Similarly:
    \begin{equation}
        (B) \le \lambda_1^g(\mathcal{I}_1^g\setminus \mathcal{I}_0^g) =\frac{\eta}{m(b-a)}.
    \end{equation}
    For the last term, notice that $\lambda_0^g$ and $\lambda_1^g$ are identical on the intersection $\mathcal{I}_0^g\cap \mathcal{I}_1^g$, since they are both uniform on intervals of the same length $L$. Hence:
    \begin{align}
        (C) &= \left|\int_{\mathcal{I}_0^g\cap \mathcal{I}_1^g} \left(\pi_0^g(\mathcal{Y}_z|z)- \pi_1^g(\mathcal{Y}_z|z)\right) \lambda_0^g(\de z)\right| \\
        &\le \int_{\mathcal{I}_0^g\cap \mathcal{I}_1^g} \left|\pi_0^g(\mathcal{Y}_z|z)- \pi_1^g(\mathcal{Y}_z|z)\right|\lambda_0^g(\de z) \\
        &\le \int_{\mathcal{I}_0^g\cap \mathcal{I}_1^g} \sup_{B\subseteq [m]}\left|\pi_0^g(B|z)- \pi_1^g(B|z)\right|\lambda_0^g(\de z) \\
        &= \int_{\mathcal{I}_0^g\cap \mathcal{I}_1^g} \tv\left(\pi_0^g(\cdot|z),\pi_1^g(\cdot|z)\right)\lambda_0^g(\de z) \\
        &\le \frac{\eta}{ma}\int_{\mathcal{I}_0^g\cap \mathcal{I}_1^g}\lambda_0^g(\de z)  &&\text{(by Equation \ref{eq:tv})} \\
        &=\frac{\eta}{ma}\lambda_0^g(\mathcal{I}_0^g\cap \mathcal{I}_1^g) \\
        &\le \frac{\eta}{ma}.
    \end{align}
    Putting the three bounds together:
    \begin{equation}
         \left|\nu_0^g(\mathcal{Y})-\nu_1^g(\mathcal{Y})\right| \le \frac{\eta}{m(b-a)} + \frac{\eta}{m(b-a)} + \frac{\eta}{ma} = \frac{\eta}{m}\left(\frac{2}{b-a}+\frac{1}{a}\right).
    \end{equation}
    Since $\mathcal{Y}$ was arbitrary:
    \begin{equation}
        \tv(\nu_0^g,\nu_1^g) = \sup_{\mathcal{Y}\subseteq \Reals\times m}\left|\nu_0^g(\mathcal{Y})-\nu_1^g(\mathcal{Y})\right| \le \frac{\eta}{m}\left(\frac{2}{b-a}+\frac{1}{a}\right).
    \end{equation}
    Finally, by the convexity of the total variation:
    \begin{equation}
        \tv(\nu_0,\nu_1) = \tv(\EE{\nu_0^G},\EE{\nu_1^G}) \le \EE{\tv(\nu_0^G,\nu_1^G)} \le \frac{\eta}{m}\left(\frac{2}{b-a}+\frac{1}{a}\right). 
    \end{equation}
    Plugging this into Equation \eqref{eq:kl} concludes the proof.
\end{proof}

\lowerbx*
\begin{proof}
    The proof is divided in several steps. Recall that an instance of \problem has $K$ base arms, $m$-sized subsets, and the action space (expressed here in terms of incidence vectors) is $\mathcal{A}=\{a\in\{0,1\}^K \mid \|a\|_1=m\}$. At each round $t$, the learner plays an action $a_t\in\mathcal{A}$. 
    We design a distribution over problem instances that is hard for all \emph{deterministic} learners. By Yao's minimax principle, this is enough to prove the existence of a hard sequence of rewards for any stochastic learner.

    \paragraph{Step 1 (game definition). } Let $n=\lfloor m/2\rfloor$. Since $m\ge 2$, $n\ge m/3$. Let $\Sigma$ be the set of permutations of $[K]$, and $\sigma\sim\unif(\Sigma)$. For $k\in[K]$, let $a_\sigma^k\in\{0,1\}^K$ be the incidence vector of the set $\{\sigma(1),\dots,\sigma(k)\}$, that is:
    \begin{equation}
        a_\sigma^k(i) = \indi{\exists j\in[k]\mid \sigma(j)=i},
    \end{equation}
    with the convention $a_\sigma^0=0$. After drawing $\sigma$ uniformly at random, the adversary selects $n$ ``good'' arms $\sigma(1),\dots,\sigma(n)$. So, $a_\sigma^n$ is the incidence vector of the set of good arms. Note that $a_\sigma^n\notin \mathcal{A}$ since $\|a_\sigma^n\|_1=n<m$.

    \paragraph{Step 2 (adversary models). }
    For each choice of $\sigma$, we consider a nominal model $P^\sigma$ and an alternative model $Q^\sigma$ where one good arm is removed. 

    Consider the nominal model $P^\sigma$ first. Independently for each round $t$, the adversary samples a uniform random variable in $[1/4, 1/2]$:
    \begin{equation*}
        U_t \sim \unif\left(\left[\frac{1}{4},\frac{1}{2}\right]\right),
    \end{equation*}
    and Bernoulli random variables $\beta_t(1),\dots,\beta_t(K)$ such that $\beta_t(i)\sim \ber(p(i))$, where
    \begin{equation}
        p(i) = \frac{1}{4} + 4\varepsilon a_\sigma^n(i)
    \end{equation}
    and $\varepsilon \in \big(0,\frac{1}{16}\big)$, so that good arms yield a slightly larger expected reward. Note that $p(i)<1/2$ since $\varepsilon < 1/16$.
    These Bernoulli variables are mutually independent among each other and with $U_t$. Then, the adversary assigns the rewards to the $K$ base arms as follows:
    \begin{equation}
        r_t(i) = \frac{1}{4}\beta_t(i) + U_t.
    \end{equation}
    Note that $r_t(i)\in[1/4, 3/4]\subset [0,1]$, as required by the problem definition.
    Also:
    \begin{equation}
        \EEp{P^\sigma}{r_t(i)} = \frac{7}{16} + \varepsilon a_\sigma^n(i).
    \end{equation}
    So, if $i$ is a good arm and $j$ is not:
    \begin{equation}
        \EEp{P^\sigma}{r_t(i)-r_t(j)} = \varepsilon.
    \end{equation}
    
    In the alternative model $Q^\sigma$, the only difference lies in the definition of the Bernoulli parameters:
    \begin{equation}
        p(i) = \frac{1}{4} + 4\varepsilon a_\sigma^{n-1}(i)    
    \end{equation}
    which amounts to canceling one of the good arms, namely $\sigma(n)$. Hence:
    \begin{equation}
        \EEp{Q^\sigma}{r_t(i)} = \frac{7}{16} + \varepsilon a_\sigma^{n-1}(i).
    \end{equation}
    In the following, for a random variable $Y$, let $P^\sigma_Y$ denote its law under the nominal model $P^\sigma$, conditional on the choice of $\sigma$, and let $\EEp{P^\sigma}{\cdot}$ denote conditional expectation $\EE{\cdot\mid\sigma}$ under $P^\sigma$. We define $Q^\sigma_Y$ and $\EEp{Q^\sigma}{\cdot}$ analogously.

    \paragraph{Step 3 (regret decomposition). }
    Let $a^*_\sigma$ be any action in $\mathcal{A}$ such that, for all $i\in[K]$, $a_\sigma^n(i)=1 \implies a^*(i)=1$, that is, any action containing all $n$ good arms. For $i\in[K]$, let $T_i$ denote the number of times the learner selects arm $i$:
    \begin{equation}
        T_i = \sum_{t=1}^T a_t(i).
    \end{equation}
    
    In the following, the expectation is with respect to the random adversary, including the choice of $\sigma$.
    Then:
    \begin{align}
        R_T &= \EE{\max_{a\in\mathcal{A}}\sum_{t=1}^T\langle a,r_t\rangle - \sum_{t=1}^T \langle a_t, r_t\rangle} \\
        &= \EE{\EEp{P^\sigma}{\max_{a\in\mathcal{A}}\sum_{t=1}^T\langle a,r_t\rangle - \sum_{t=1}^T \langle a_t, r_t\rangle}} \\
        &\ge \EE{\EEp{P^\sigma}{\sum_{t=1}^T \langle a^*_\sigma - a_t, r_t\rangle}} \\
        &= \varepsilon\EE{\EEp{P^\sigma}{\sum_{j=1}^n (T-T_{\sigma(j)})}} \\
        &= \varepsilon\left(nT - \sum_{j=1}^n\EE{\EEp{P^\sigma}{T_{\sigma(j)}}}\right).
    \end{align}

    Let $M\coloneqq \EE{\EEp{P^\sigma}{T_{\sigma(1)}}}$. Since good arms are indistinguishable and $\sigma$ is sampled uniformly at random, by symmetry,
    \begin{equation}
        \EE{\EEp{P^\sigma}{T_{\sigma(j)}}} = M \qquad \text{for all } j \in[n].
    \end{equation}
    Hence,
    \begin{equation}
        R_T \ge \varepsilon n\left(T-M\right).\label{eq:new_regret}
    \end{equation}

    \paragraph{Step 4 (information-theoretic argument). }
    Let $M'=\EE{{\EEp{Q^\sigma}{T_{\sigma(n)}}}}$, the expected number of times the "canceled" good arm is played under the alternative model. Since bad arms are indistinguishable and $\sigma$ is uniformly random, by symmetry,
    \begin{equation}
        \EE{{\EEp{Q^\sigma}{T_{\sigma(j)}}}} = M' \qquad \text{for } j=n,\dots,K.
    \end{equation}
    Hence:
    \begin{equation}
        (K-n+1)M' = \sum_{j=n}^K \EE{{\EEp{Q^\sigma}{T_{\sigma(j)}}}} \le mT,
    \end{equation}
    which implies:
    \begin{equation}\label{eq:mprime}
        M' \le \frac{mT}{K-n+1}.
    \end{equation}

    Now fix $\sigma$. Let $Y_t=(r_t(S_t), I_t)$ denote the observation of the learner at round $t$, and $Y_{1:t}=(Y_1,\dots,Y_t)$ the sequence of observations up to round $t$ included. Since $T_j\in\{0,\dots,T\}$:
    \begin{align}
        \EEp{P^\sigma}{T_{\sigma(n)}} - \EEp{Q^\sigma}{T_{\sigma(n)}} 
        &\le T \tv\left(P^\sigma_{T_{\sigma(n)}},Q^\sigma_{T_{\sigma(n)}}\right) \\
        &\le T\sqrt{\frac{1}{2}\kl\left(Q^\sigma_{T_{\sigma(n)}}\|P^\sigma_{T_{\sigma(n)}}\right)} \label{eq:pinsker2}\\ 
        &\le T\sqrt{\frac{1}{2}\kl\left(Q^\sigma_{Y_{1:T}}\|P^\sigma_{Y_{1:T}}\right)}\label{eq:data} \\
        &= T\sqrt{\frac{1}{2}\sum_{t=1}^T \EEp{Q^\sigma}{\kl\left(Q^\sigma_{Y_t\mid Y_{1:t-1}}\|P^\sigma_{Y_t\mid Y_{1:t-1}}\right)}},\label{eq:chain2}
    \end{align}
    where Equation \eqref{eq:pinsker2} is by Pinsker's inequality; Equation \eqref{eq:data} is by the data-processing inequality, since, for a deterministic learner, $T_{\sigma(n)}$ is a measurable function of the observations; the last equality is by the chain rule of the KL divergence. 

    Now, fix a round $t$ and consider the KL divergence in Equation \eqref{eq:chain2}. Conditional on $Y_{1:t-1}$, the deterministic learner selects the same action under $P^\sigma$ and $Q^\sigma$: the conditional observation laws differ (and contribute to the KL) only if this action contains $\sigma(n)$. Hence:
    \begin{align}
        \EEp{Q^\sigma}{\kl\left(Q^\sigma_{Y_t\mid Y_{1:t-1}}\|P^\sigma_{Y_t\mid Y_{1:t-1}}\right)} 
        &= \EEp{Q^\sigma}{a_t(\sigma(n))\kl\left(Q^\sigma_{Y_t\mid Y_{1:t-1}}\|P^\sigma_{Y_t\mid Y_{1:t-1}}\right)}.
    \end{align}
     So, we assume $a_t(\sigma(n))=1$ and bound the KL divergence.
     Let $B_1,\dots,B_m$ denote a re-indexing of the Bernoulli draws $\beta_i$ for $i\in S_t$, the base arms selected by the learner at the round of interest, such that $\sigma(n)$ corresponds to $B_1$. We can apply Lemma \ref{lem:kl} with $a=1/4$, $b=1/2$, $\eta=1/4$, $p=1/4$, $\delta=4\varepsilon<1/4$, and $q_i\in\{1/4,1/4+4\varepsilon\}$ are the Bernoulli parameters of the other $m-1$ selected arms, which are identical under $P^\sigma$ and $Q^\sigma$. Then:
    \begin{equation}\label{eq:kl384}
        \kl\left(Q^\sigma_{Y_t\mid Y_{1:t-1}}\|P^\sigma_{Y_t\mid Y_{1:t-1}}\right) \le a_t(\sigma(n))\frac{384 \varepsilon^2}{m}. 
    \end{equation}
    Note that, since $m\ge 2$, $\eta = 1/4 < m/4 = m(b-a)$, as required by the lemma. Plugging this into Equation \eqref{eq:chain2}:
    \begin{align}
        \EEp{P^\sigma}{T_{\sigma(n)}} - \EEp{Q^\sigma}{T_{\sigma(n)}} 
        &\le 8T\varepsilon\sqrt{\frac{3}{m}\sum_{t=1}^T\EEp{Q^\sigma}{a_t(\sigma(n))}} \\
        &=8T\varepsilon\sqrt{\frac{3}{m}\EEp{Q^\sigma}{T_{\sigma(n)}}}.
    \end{align}
    Averaging over $\sigma$, by Jensen's inequality:
    \begin{equation}
         \EE{\EEp{P^\sigma}{T_{\sigma(n)}}} - \EE{\EEp{Q^\sigma}{T_{\sigma(n)}}} 
         \le \EE{8T\varepsilon\sqrt{\frac{3}{m}\EEp{Q^\sigma}{T_{\sigma(n)}}}}
         \le 8T\varepsilon\sqrt{\frac{3}{m}\EE{\EEp{Q^\sigma}{T_{\sigma(n)}}}}.
    \end{equation}
    Equivalently:
    \begin{equation}
        M \le M' + 8T\varepsilon\sqrt{\frac{3M'}{m}}.
    \end{equation}
    By Equation \eqref{eq:mprime}:
    \begin{equation}
        M \le \frac{mT}{K-n+1} + 8T\varepsilon\sqrt{\frac{3T}{K-n+1}}.\label{eq:final}
    \end{equation}

    \paragraph{Step 5 (regret lower bound). }
    Lemma \ref{lem:kl}, as applied in the previous step, also ensures that, under both models, the rewards $r_t(i)=X_i$ for $i\in S_t$ are a.s.\ in $[1/4, 3/4]\subseteq [0,1]$, as required by the \problem problem definition. This also ensures that $r_t(S_t)>0$ almost surely. 
    
    Let $\kappa = K-n+1$. Since $K\ge 2m$ and $n\le m/2$:
    \begin{equation}
        \frac{K-n+1}{m} \ge \frac{2m - \frac{m}{2} + 1}{m} \ge \frac{3}{2},
    \end{equation}
    so $\kappa\ge 3/2 m$.
    By plugging the upper bound from Equation \eqref{eq:final} into Equation \eqref{eq:new_regret}, we obtain:
    \begin{equation}
        R_T \ge \varepsilon nT\left(1-\frac{m}{\kappa} - 8\varepsilon\sqrt{\frac{3T}{\kappa}}\right) \ge \varepsilon nT\left(\frac{1}{3} - 8\varepsilon\sqrt{\frac{3T}{\kappa}}\right).
    \end{equation}
    Solving for $\varepsilon$:
        \begin{equation}
            R_T \ge \frac{n}{288\sqrt{3}}\sqrt{\kappa T},
        \end{equation}
        for $\varepsilon=\frac{1}{48}\sqrt{\frac{\kappa}{3T}}\le \frac{1}{48}\sqrt{\frac{K}{3T}}$. Since $T>K/27$, $\varepsilon<1/16$ as required by the previous steps. Finally, recall that $m/3\le n\le m/2$ and $K\ge 2m$, so $\kappa = K-n+1\ge K-m/2 \ge 3/4 K$, so:
        \begin{equation}
            R_T \ge \frac{m\sqrt{K T}}{1728}.
        \end{equation}
    \end{proof}

\section{Proofs of Section~\ref{sec:other}}\label{app:profit}
\linearlb*
\begin{proof}
    We partition the arm set $[K]$ into \emph{special} arms $1$ and $2$ and \emph{common} arms $C=\{3,\dots,K\}$. Let $n\coloneqq |C| = K-2$ for brevity. Define
    \begin{align*}
        \alpha = \frac{n-m+1}{n},\qquad \beta = \frac{m(m-1)}{2n}, \qquad Z = 1 + \alpha + \beta. 
    \end{align*}
    Fix a round $t$ and consider two probability models $P,Q$. Under $P$:
    \begin{itemize}
        \item With probability $1/Z$, let $r_t(1)=0$, $r_t(2)=1$, and assign reward $1$ to one arm sampled uniformly at random from $C$, zero reward to the others; we call this event $E$;
        \item With probability $\alpha/Z$, let $r_t(1)=1$, $r_t(2)=0$, and $r_t(i)=0$ for all $i\in C$; we call this event $A$;
        \item With probability $\beta/Z$, let $r_t(1)=1$, $r_t(2)=0$, and $r_t(i)=1$ for all $i\in C$. We call this event $B$.
    \end{itemize}
    Under $Q$, the rewards of arms $1$ and $2$ are exchanged. We are going to prove that $Y_t=(I_t, r_t(I_t))$ has the same law under $P$ and $Q$. If $S_t$ does not contain arm $1$ nor $2$, this is obvious. We are going to consider three remaining cases.
    
    \paragraph{Case 1: $1\in S_t$, $2\notin S_t$.}
    There must be $s\coloneqq m-1$ common arms in $S_t$ ($|S_t\cap C|=s$). Let $X$ be the number of active arms (arms with a reward of one) in $S_t\cap C$. On the event $E$, $C$ contains exactly one active arm. So, regardless of the choice of $S_t$, due to the uniform sampling, $S_t\cap C$ contains an active arm with probability $s/n$. Then, under $P$, the probability that the learner gets zero total reward is:
    \begin{align*}
        P(r_t(S_t)=0) = \frac{1}{Z}P(r_t(S_t)=0\mid E) + \frac{\alpha}{Z}P(r_t(S_t)=0\mid A) + \frac{\beta}{Z}P(r_t(S_t)=0\mid B) = \frac{1}{Z}\left(1-\frac{s}{n}\right),
    \end{align*}
    since, under $A$ and $B$, the learner gains positive reward from arm $1\in S_t$. Similarly, under $Q$:
    \begin{align*}
        Q(r_t(S_t)=0) = \frac{1}{Z} \cdot 0 + \frac{\alpha}{Z} \cdot 1 + \frac{\beta}{Z} \cdot 0 = \frac{1}{Z}\frac{n-m+1}{n} = \frac{1}{Z}\left(1-\frac{s}{n}\right) = P(r_t(S_t)=0).
    \end{align*}
    When $r_t(S_t)=0$, all arms have the same probability of being the winner, and this is the only case in which a winner can be observed with a reward of zero. So, for all $i\in S_t$, $P(Y_t=(i, 0))=Q(Y_t=(i, 0))$. Now consider the probability of observing arm $1$ as the winner with a reward of one. Under $P$:
    \begin{align*}
        P(Y_t=(1,1)) = \frac{1}{Z}\cdot 0 + \frac{\alpha}{Z}\cdot 1 + \frac{\beta}{Z} \cdot \frac{1}{m} = \frac{1}{Z}\left(\alpha + \frac{\beta}{m}\right) = \frac{1}{Z}\left(1-\frac{s}{2n}\right),
    \end{align*}
    since under event $E$ arm $1$ cannot be the winner, under event $A$ it is the only active arm, and under $B$ it competes with other $s=m-1$ active arms. Under $Q$, $E$ is the only event in which arm $1$ can be the winner. If $S_t\cap C$ does not contain an active arm, arm $1$ is automatically the winner, otherwise it has $1/2$ probability. So:
    \begin{align*}
        Q(Y_t=(1,1)) = \frac{1}{Z}\left(\left(1-\frac{s}{n}\right)\cdot 1 + \frac{s}{n} \cdot \frac{1}{2}\right) + \frac{\alpha}{Z} \cdot 0 + \frac{\beta}{Z} \cdot 0 = \frac{1}{Z}\left(1-\frac{s}{2n}\right) = P(Y_t=(1,1)). 
    \end{align*}
    By normalization and the exchangeability of the common arms, also $P(Y_t=(i,1))=Q(Y_t=(i,1))$ for all $i\in C$, so $Y_t$ has the same law under $P$ and $Q$.

    \paragraph{Case 2: $1\notin S_t$, $2\in S_t$. } As in Case 1, by symmetry.

    \paragraph{Case 3: $1,2\in S_t$. } In this case $S_t$ contains $s'=m-2$ common arms. Since it always contains both special arms, $P(r_t(S_t)=0) = Q(r_t(S_t)=0) = 0$, so also $P(Y_t=(i,0)) = 0 = Q(Y_t=(i,0))$ for all $i\in S_t$. As for arm $1$:
    \begin{align*}
        P(Y_t=(1,1)) = \frac{1}{Z} \cdot 0 + \frac{\alpha}{Z} \cdot 1 + \frac{\beta}{Z} \cdot \frac{1}{m-1} = \frac{1}{Z}\left(1-\frac{s'}{2n}\right),
    \end{align*}
    since, under event $B$, arm $1$ is competing with other $s'=m-2$ active arms. Also:
    \begin{align*}
        Q(Y_t=(1,1)) = \frac{1}{Z}\left(\left(1-\frac{s'}{n}\right)\cdot 1 + \frac{s'}{n}\cdot \frac{1}{2}\right) + \frac{\alpha}{Z}\cdot 0 + \frac{\beta}{Z}\cdot 0 = \frac{1}{Z}\left(1-\frac{s'}{2n}\right) = P(Y_t=(1,1)).
    \end{align*}
    The same calculation with the roles of arms $1$ and $2$ exchanged shows that $P(Y_t=(2,1))=Q(Y_t=(2,1))$.
    Hence, by exchangeability of the selected common arms and normalization, $Y_t$ has the same law under $P$ and $Q$.

    We have shown that $Y_t$ has the same law under $P$ and $Q$ conditional on any valid $S_t$. 

    Next, we compute the expected reward of all arms. Under $P$:
    \begin{equation}\label{eq:abc}
    \begin{aligned}
        &\EEp{P}{r_t(1)} = \frac{\alpha+\beta}{Z} \coloneqq a,\\
        &\EEp{P}{r_t(2)} = \frac{1}{Z} \coloneqq b,\\
        &\EEp{P}{r_t(i)} = \frac{1}{Z}\left(\frac{1}{n}+\beta\right) \coloneqq c \qquad \text{for all } i\in C.
    \end{aligned}
    \end{equation}
    Note that $a,b,c<1$. Also:
    \begin{equation}\label{eq:abcdiff}
    \begin{aligned}
        &a-b = \frac{(m-2)(m-1)}{2nZ} > 0, \\
        &a-c = \frac{n-m}{nZ} \ge \frac{m-2}{nZ} > 0,
    \end{aligned}
    \end{equation}
    since $n=K-2\ge 2m-2$ and $m\ge 3$. Under $Q$, $a$ and $b$ are exchanged.
    
    Now let $G\subset C$ with $|G|=m-1$ be an arbitrary proper subset of good common arms, and let $\delta_G$ be the probability model that assigns reward $1$ to arms $i\in G$ and zero to all other arms. Let
    \begin{equation*}
        P_G = \frac{3}{4}\delta_G + \frac{1}{4}P, \qquad Q_G = \frac{3}{4}\delta_G + \frac{1}{4}Q.
    \end{equation*}
    Clearly, $Y_t$ has the same law under $P_G$ and $Q_G$ too. Also, from Equation~\eqref{eq:abc}:
    \begin{align*}
        &\EEp{P_G}{r_t(1)} = \frac{a}{4} < \frac{1}{4}, \\
        &\EEp{P_G}{r_t(2)} = \frac{b}{4} < \frac{1}{4}, \\
        &\EEp{P_G}{r_t(i)} = \frac{3}{4} + \frac{c}{4} \ge \frac{3}{4} \qquad \text{for all } i\in G, \\
        &\EEp{P_G}{r_t(i)} = \frac{c}{4} < \frac{1}{4} \qquad \text{for all } i\in C\setminus G,
    \end{align*}
    while $a$ and $b$ are exchanged under $Q_G$. By Equation~\eqref{eq:abcdiff}, the optimal composite actions under the two models are:
    \begin{equation}
        S^*_{P_G} = G \cup \{1\}, \qquad S^*_{Q_G} = G \cup \{2\}.
    \end{equation}
    Let $g\coloneqq (3+c)/4$. Since $a$ is the value of the optimal special arm under each model:
    \begin{equation}\label{eq:halfopt}
        \frac{\EEp{P_G}{r_t(S^*_{P_G})} + \EEp{Q_G}{r_t(S^*_{Q_G})}}{2} = (m-1)g+\frac{a}{4}. 
    \end{equation}
    Moreover,
    \begin{align*}
        \frac{\EEp{P_G}{r_t(i)}+\EEp{Q_G}{r_t(i)}}{2} = \begin{cases}
            g & i\in G \\
            \frac{c}{4} & i\in C\setminus G \\
            \frac{a+b}{8} &i \in \{1,2\}.
        \end{cases}
    \end{align*}
    Hence, for any $S\subseteq [K]$ such that $|S|=m$:
    \begin{equation}\label{eq:half}
        \frac{\EEp{P_G}{r_t(S)}+\EEp{Q_G}{r_t(S)}}{2} \le (m-1)g + \max\left\{\frac{c}{4},\frac{a+b}{8}\right\}.
    \end{equation}
    Putting Equations~\eqref{eq:halfopt} and~\eqref{eq:half} together:
    \begin{align}
        \frac{1}{2}\left(\EEp{P_G}{r_t(S^*_{P_G})-r_t(S)} + \EEp{Q_G}{r_t(S^*_{Q_G})-r_t(S)}\right)
        &\ge \frac{a}{4} - \max\left\{\frac{c}{4},\frac{a+b}{8}\right\} \nonumber\\
        &= \frac{1}{4}\min\left\{a-c, \frac{a-b}{2}\right\} \coloneqq \gamma > 0,\label{eq:gamma}
    \end{align}
    where positivity is by Equation~\eqref{eq:abcdiff}.

    Before the game starts, the adversary fixes the set of good common arms $G$ and draws the sequence of rewards from either $\mathbb{P}=P_G^{\otimes T}$ or $\mathbb{Q}=Q_G^{\otimes T}$. Since the conditional law of $Y_t$ is identical under $P_G$ and $Q_G$ for any action of the learner, the whole feedback sequence, and hence the learner's actions, have the same law under $\mathbb{P}$ and $\mathbb{Q}$. So, summing Equation~\eqref{eq:gamma} over time:
    \begin{equation*}
        \frac{\EEp{\mathbb{P}}{R_T}+\EEp{\mathbb{Q}}{R_T}}{2} \ge \gamma T.
    \end{equation*}
    This means that at least one of the two models induces an expected regret of at least $\gamma T$. Since each is a distribution over deterministic oblivious reward sequences, averaging yields a deterministic oblivious reward sequence on which the learner's regret is at least $\gamma T$. The theorem's statement is obtained by substituting the values of $a,b,c,Z,n$ in the expression of $\gamma$.
\end{proof}

\lowermulti*
\begin{proof}
    Consider the same construction and regret composition introduced in Steps 1--3 of the proof of Theorem~\ref{th:lower2}. These steps do not depend on the feedback mechanism, so they carry on without modifications. What changes is the KL upper bound from Step 4. Fix permutation $\sigma$, round $t$, and the past observations $Y_{1:t-1}$, where $Y_t=(J_t, r_t(J_t))$, differently from the original proof. 
    Conditional on the past, action $S_t$ is fixed.
    The reward distributions differ only on arm $\sigma(n)$. Hence, by Lemma~\ref{lem:kluniform}:
    \begin{align*}
        \kl(Q^\sigma_{Y_t|Y_{1:t-1}}\|P^\sigma_{Y_t|Y_{1:t-1}}) 
        &= \frac{1}{m}\sum_{i\in S_t}\kl(Q^\sigma_{r_t(i)}\|P^\sigma_{r_t(i)}) \\
        &= \frac{\indi{\sigma(n)\in S_t}}{m}\kl(Q^\sigma_{r_t(\sigma(n))}\|P^\sigma_{r_t(\sigma(n))}) \\
        &\le \frac{\indi{\sigma(n)\in S_t}}{m}\kl(Q^\sigma_{\beta_t(\sigma(n))}\|P^\sigma_{\beta_t(\sigma(n))}) \\
        &=\frac{\indi{\sigma(n)\in S_t}}{m}\klber\left(\frac{1}{4}\middle\|\frac{1}{4}+4\varepsilon\right),
    \end{align*}
    by the data-processing inequality, since $r_t(i)=\frac{1}{4}\beta_t(i)+U_t$ and $U_t$ is the same under the two models. Here $\klber(p\|q) \coloneqq \kl(\ber(p)\|\ber(q)) = p\log(p/q)+(1-p)\log((1-p)/(1-q))$ denotes the KL divergence of two Bernoulli distributions with parameters $p,q$. 
    Let $f(q)=\klber(1/4\| q)$. Since $f(1/4)=f'(1/4)=0$ and $\varepsilon < 1/16$, Taylor's theorem yields:
    \begin{equation*}
        f\left(\frac{1}{4}+4\varepsilon\right) \le \frac{1}{2}\sup_{x\in[1/4,1/2]}f''(x)(4\varepsilon)^2 = 8 \varepsilon^2\sup_{x\in[1/4,1/2]}f''(x).
    \end{equation*}
    Since $f''$ is convex on $[1/4,1/2]$, $\sup_{x\in[1/4,1/2]}f''(x)\le \max\{f''(1/4),f''(1/2)\}=16/3$. Hence:
    \begin{equation*}
        \klber\left(\frac{1}{4}\middle\|\frac{1}{4}+4\varepsilon\right) \le \frac{128}{3}\varepsilon^2,
    \end{equation*}
    and
    \begin{equation}\label{eq:kluniform}
    \kl(Q^\sigma_{Y_t|Y_{1:t-1}}\|P^\sigma_{Y_t|Y_{1:t-1}}) \le a_t(\sigma(n))\frac{128\varepsilon^2}{3m}.
    \end{equation}
    Equation~\eqref{eq:kluniform} replaces Equation~\eqref{eq:kl384} in Step 4 of the proof of Theorem~\ref{th:lower2}. The proof continues in the same way, just for different constants. The choice of $\varepsilon$ is $\frac{1}{48}\sqrt{\frac{3\kappa}{T}}$ which requires $T>K/3$ to ensure $\varepsilon<1/16$ and yields
    \begin{equation*}
        R_T \ge \frac{n\sqrt{3\kappa T}}{288} \ge \frac{m\sqrt{KT}}{576}.
    \end{equation*}
\end{proof}

\profitlb*
\begin{proof}
    Let $\{1,2\}$ be \emph{core} arms and let $D=\{3,\dots,K\}$ be \emph{dummy} arms that always have zero reward in all rounds.
    Since $K>m$, $|D|=K-2\geq m-1$. Consequently, there exist valid actions $S$ satisfying
    \begin{equation*}
        S\cap\{1,2\}=\{1\},
        \qquad
        S\cap\{1,2\}=\{2\},
        \qquad\text{or}\qquad
        S\cap\{1,2\}=\{1,2\}.
    \end{equation*}
    Fix a round $t$ and define two probability models $P,Q$ for $(r_t(1), r_t(2))$ as follows:
    \begin{equation}
    \begin{array}{c|cccccc}
        (r_t(1),r_t(2))
        &(0,\frac23)&(\frac13,\frac23)&(\frac13,1)
        &(\frac23,0)&(\frac23,\frac13)&(1,\frac13)\\ \hline
        P
        &\frac16&\frac16&\frac16&\frac16&\frac16&\frac16 \\
    \end{array}
    \end{equation}
    and
    \begin{equation}
    \begin{array}{c|cccc}
        (r_t(1),r_t(2))
        &(0,1)&(\frac13,\frac13)&(\frac23,\frac23)&(1,0)\\ \hline
        Q
        &\frac16&\frac13&\frac13&\frac16
    \end{array}
    \end{equation}
    We show that feedback $Y_t=(r_t(S_t), I_t)$ has the same law, conditional on $S_t$, under the two models. First, consider the individual rewards of the core arms. For $i\in\{1,2\}$:
    \begin{align*}
        &P(r_t(i)=0) = Q(r_t(i)=0) = \frac{1}{6}, \\ 
        &P(r_t(i)=1) = Q(r_t(i)=1) = \frac{1}{6}, \\
        &P\left(r_t(i)=\frac{1}{3}\right) = Q\left(r_t(i)=\frac{1}{3}\right) = \frac{1}{3}, \\
        &P\left(r_t(i)=\frac{2}{3}\right) = Q\left(r_t(i)=\frac{2}{3}\right) = \frac{1}{3},
    \end{align*}
    So the marginal law of each core-arm reward is the same under the two models.
    The law of $Y_t=(r_t(S_t), I_t)$ depends on how many core arms are contained in $S_t$. We consider the three cases separately.

    \paragraph{Case 1: $|\{1,2\}\cap S_t|=1$. } $S_t$ contains exactly one core arm $j\in\{1,2\}$. Total reward $r_t(S_t)$ equals the reward of the core arm $r_t(j)$. If $r_t(j)>0$, $j$ is the winner. Otherwise, $I_t\sim \unif(S_t)$. In both cases, the law of $Y_t$ is the same under $P$ and $Q$.

    \paragraph{Case 2: $|\{1,2\}\cap S_t|=2$. } $S_t$ contains both core arms. The only possible values of the total reward are $r_t(S_t)\in\{2/3, 1, 4/3\}$. Since they are all positive, the winner arm is always a core arm. 
    Direct computation shows that, for any valid $S\supseteq \{1,2\}$:
    \begin{equation*}
        P(Y_t=(\lambda, i)|S_t=S) = Q(Y_t=(\lambda, i)|S_t=S) = \frac{1}{6} \qquad \text{for all } \lambda\in\left\{\frac{2}{3},1,\frac{4}{3}\right\} \text{ and } i\in\{1,2\}. 
    \end{equation*}

    \paragraph{Case 3: $\{1,2\}\cap S_t=\emptyset$. } $S_t$ only contains dummy arms, so $r_t(S_t)=0$ and $I_t\sim\unif(S_t)$ under both models.

    We have proven that $P_{Y_t|S_t}=Q_{Y_t|S_t}$. Next, we compute expected winner rewards. Since dummy arms have reward zero, $\rho_t(S)$ only depends on the core set $\{1,2\}\cap S$. Hence, with slight abuse of notation, we write $\rho_t(\{1,2\}\cap S)$ in place of $\rho_t(S)$. Direct computation shows:
    \begin{align*}
        &\EEp{P}{\rho_t(\{1\})} = \EEp{Q}{\rho_t(\{1\})} = \frac{1}{2},\\
        &\EEp{P}{\rho_t(\{2\})} = \EEp{Q}{\rho_t(\{2\})} = \frac{1}{2},\\
        &\EEp{P}{\rho_t(\{1,2\})} = \frac{37}{54},\qquad
        \EEp{Q}{\rho_t(\{1,2\})} = \frac{2}{3}.
    \end{align*}

    Now let $\delta$ be the point mass at $(1,1/3)$ and let:
    \begin{equation*}
        P_+ = \frac{19}{37}\delta + \frac{18}{37}P, \qquad
        Q_+ = \frac{19}{37}\delta + \frac{18}{37}Q.
    \end{equation*}

    Clearly, the law of $Y_t$ is also the same under $P_+$ and $Q_+$. As for the expected winner reward:
    \begin{equation*}
        \EEp{\delta}{\rho_t(\{1\})} = 1, \qquad \EEp{\delta}{\rho_t(\{2\})} = \frac{1}{3}, \qquad \EEp{\delta}{\rho_t(\{1,2\})} = \frac{5}{6}.
    \end{equation*}
    So, the expected winner rewards under $P_+,Q_+$ are as follows. For any valid $S$:
        \begin{equation*}
    \begin{array}{c|cccc}
        &\{1,2\}\cap S=\emptyset
        &\{1,2\}\cap S=\{1\}
        &\{1,2\}\cap S=\{2\}
        &\{1,2\}\cap S=\{1,2\}\\ \hline
        \EEp{P_+}{\rho_t(S)}
        &0
        &\frac{168}{222}
        &\frac{92}{222}
        &\frac{169}{222}\\[.5em]
        \EEp{Q_+}{\rho_t(S)}
        &0
        &\frac{168}{222}
        &\frac{92}{222}
        &\frac{167}{222}
    \end{array}
    \end{equation*}
    Then, the optimal composite action under $P_+$ contains both core arms, while the optimal composite action under $Q_+$ only contains arm $1$. Optimal actions are not unique, but we can fix an optimal action $S^*_{P_+}$ under $P_+$ and an optimal action $S^*_{Q_+}$ under $Q_+$.
    In both cases, the minimum gap with respect to a suboptimal action is $1/222$. Hence, for any fixed valid action $S$:
    \begin{equation}\label{eq:444}
        \frac{1}{2}\left(\EEp{P_+}{\rho_t(S^*_{P_+})} - \EEp{P_+}{\rho_t(S)}+\EEp{Q_+}{\rho_t(S^*_{Q_+})} - \EEp{Q_+}{\rho_t(S)}\right) \ge \frac{1}{444},
    \end{equation}
    since either $S$ contains both core arms, in which case the regret under $Q_+$ is $1/222$, or it does not, in which case the regret under $P_+$ is at least $1/222$.

    Before the game starts, the adversary draws the sequence of rewards from either $\mathbb{P}=P_+^{\otimes T}$ or $\mathbb{Q}=Q_+^{\otimes T}$. Since the conditional law of $Y_t$ is identical under $P_+$ and $Q_+$ for any action of the learner, the whole feedback sequence, and hence the learner's actions, have the same law under $\mathbb{P}$ and $\mathbb{Q}$. So, summing Equation~\eqref{eq:444} over time:
    \begin{equation*}
        \frac{\EEp{\mathbb{P}}{R_T}+\EEp{\mathbb{Q}}{R_T}}{2} \ge \frac{T}{444}.
    \end{equation*}
    This means that at least one of the two models induces an expected regret of at least $T/444$. Since each is a distribution over deterministic oblivious reward sequences, averaging yields a deterministic oblivious reward sequence on which the learner's regret is at least $T/444$.
\end{proof}

\winwin*
\begin{proof}
    Let $M={K\choose m}$. We prove the upper and lower bound separately.

    \paragraph{Upper bound. }
    Consider each $S\subseteq[K]$ with $|S|=m$ as an arm of an adversarial $M$-armed bandit.
    Without loss of generality, assume that the winner indexes $J_t(S)\sim q_t(\cdot | S)$ and the reward vectors $r_t$ are generated before the game starts for every arm $S$ and all $t\in[T]$. 
    Let $X_t(S)=r_t(J_t(S))$ denote the stochastic reward of arm $S$.
    Note that, for every arm $S$, $X_t(S)\in[0,1]$ and $\EE{X_t(S)}=\rho_t(S)$.
    Upon playing $S_t$, the agent observes and gains $X_t(S_t)$, which is distributed exactly as $r_t(I_t)$. The observed index $I_t$ is ignored.
    Run INF~\citep{AudibertB10} on the $M$ arms with losses $\ell_t(S)=1-X_t(S)$. By~\citep[][Theorem 11]{AudibertB10}, $R_T\le 8\sqrt{MT}$.

    \paragraph{Lower bound. }
    By Yao's minimax principle, it is enough to construct a distribution over oblivious reward sequences under which every deterministic learner suffers the desired regret.

    Consider a reference model $P$ where base-arm rewards are $r_t(i)\sim\ber(1/2)$ independently for all $i\in[K]$ and $t\in[T]$. Fix $\varepsilon\in(0,1/2]$. For every valid action $S\subseteq[K]$ with $|S|=m$, let $P^S$ be an alternative model where $r_t(i)\sim \ber(1/2)$ for all $i\notin S$, and:
    \begin{equation*}
        P^S(v_t(S)=x) = 2^{-m}\left(1-\varepsilon(-1)^{\sum_{i=1}^m x_i}\right), \qquad x\in\{0,1\}^m,
    \end{equation*}
    where $v_t(S)\in \{0,1\}^m$ denotes the vector of the rewards of the base arms in $S$.
    To see that these probabilities sum to one, note that the parity terms cancel out when summing over $x\in\{0,1\}^m$.
    The rewards outside $S$ are mutually independent and independent of $v_t(S)$.
    
    For any proper subset $A\subset S$ and $y\in\{0,1\}^{|A|}$,
    \begin{align}
        P^S(v_t(A)=y) &= \sum_{x\in\{0,1\}^{m-|A|}}2^{-m}\left(1-\varepsilon(-1)^{\sum_{i=1}^{m-|A|} x_i + \sum_{i=1}^{|A|}y_i}\right) \nonumber \\
        &=2^{m-|A|-m} - 2^{-m}\varepsilon(-1)^{\sum_{i=1}^{|A|}y_i}\sum_{x\in\{0,1\}^{m-|A|}}(-1)^{\sum_{i=1}^{m-|A|} x_i} \nonumber \\
        &= 2^{-|A|}, \label{eq:proper}
    \end{align}
    since the parity terms cancel out. The latter does not depend on $S$. So, whenever $S'\neq S$, $S'\cap S$ is a proper subset of $S$ and $v_t(S')$ has the same distribution under $P$ and $P^S$.
    
    Fix a valid action $A$. Under both $P$ and $P^S$, the distribution of $v_t(A)$ is invariant under permutation of the base arms in $A$. Indeed, under $P$ base rewards are i.i.d. Under $P^S$, when $A\neq S$, this follows from Equation~\eqref{eq:proper}, while, when $A=S$, the probability of the reward vector depends only on its parity. If $r_t(I_t)=0$, all base arms in $A$ have reward zero, so $I_t$ is uniform on $A$. If $r_t(I_t)=1$, for any model $Q\in\{P,P^S\}$ and any $i\in A$:
    \begin{equation*}
        Q(I_t=i, r_t(I_t)=1|S_t=A) = \EEp{Q}{\frac{r_t(i)}{r_t(A)}\indi{r_t(A)>0}}.
    \end{equation*}
    By permutation invariance, this probability is the same for every $i$. Moreover:
    \begin{equation*}
        \sum_{i\in A}Q(I_t=i,r_t(I_t)=1|S_t=A) = \EEp{Q}{\frac{\sum_{i\in A}r_t(i)}{r_t(A)}\indi{r_t(A)>0}|S_t=A} = Q(r_t(I_t)=1|S_t=A).
    \end{equation*}
    Since they are all the same,
    \begin{equation*}
        Q(I_t=i,r_t(I_t)=1|S_t=A) = \frac{1}{m}Q(r_t(I_t)=1|S_t=A).
    \end{equation*}
    So, conditional on the winner's reward, the winner's index is uniform on $S_t$.

    Since the rewards are binary, $r_t(I_t)=\indi{r_t(S_t)>0} = 1-\indi{v_t(S_t)=\mathbf{0}}$. Under $P$, $\rho_t(S)=1-P(v_t(S)=\mathbf{0}) = 1-2^{-m}$ for all valid $S$. Under $P^S$, $\rho_t(S')=1-2^{-m}$ for all $S'\neq S$. However, $\rho_t(S)=1-P^S(v_t(S)=\mathbf{0}) = 1-2^{-m}(1-\varepsilon)$. Hence, under $P^S$, $S$ is the unique optimal action, with a gap of $2^{-m}\varepsilon$ with respect to any other action. Let $T_S=\sum_{t=1}^T\indi{S_t=S}$ be the number of times the learner plays action $S$. Then:
    \begin{equation}\label{eq:gap}
        \EEp{P^S}{R_T} = 2^{-m}\varepsilon(T-\EEp{P^S}{T_S}).
    \end{equation}
    Let $\mathbb{P}$ and $\mathbb{P}^S$ denote the laws of the complete feedback sequence under $P$ and $P^S$, respectively. By the chain rule of the KL divergence,
    \begin{equation*}
        \kl(\mathbb{P}\|\mathbb{P}^S) = \sum_{t=1}^T\EEp{P}{\kl\left(P_{Y_t|Y_{1:t-1}}\|P^S_{Y_t|Y_{1:t-1}}\right)}.
    \end{equation*}
    Conditional on $Y_{1:t-1}$, $S_t$ is fixed. If $S_t\neq S$, $Y_t$ has the same law under the two models and the KL is zero. If $S_t=S$, the two conditional laws are $\unif(S)\otimes\ber(1-2^{-m})$ and $\unif(S)\otimes\ber(1-2^{-m}(1-\varepsilon))$. Since the uniform factor is identical under the two models:
    \begin{align}
        \kl(\mathbb{P}\|\mathbb{P}^S) 
        &= \sum_{t=1}^T\EEp{P}{\indi{S_t=S}}\klber(1-2^{-m}\|1-2^{-m}(1-\varepsilon)) \nonumber\\
        &=\EEp{P}{T_S}\klber(1-2^{-m}\|1-2^{-m}(1-\varepsilon))\nonumber\\
        &\le \EEp{P}{T_S}\frac{(2^{-m}\varepsilon)^2}{(1-2^{-m}(1-\varepsilon))2^{-m}(1-\varepsilon)} \nonumber\\
        &=\EEp{P}{T_S}\frac{2^{-m}\varepsilon^2}{(1-2^{-m}(1-\varepsilon))(1-\varepsilon)} \nonumber\\
        &\le 2^{-m}\EEp{P}{T_S}\frac{8\varepsilon^2}{3}, \label{eq:83}
    \end{align}
    where the first inequality is by upper bounding the KL by the Chi-squared divergence and the last one is by $1-\varepsilon\ge 1/2$ and $1-2^{-m}(1-\varepsilon)\ge 1-2^{-m}\ge 3/4$ since $\varepsilon\le1/2$ and $m\ge 2$. By Pinsker's inequality and Equation~\eqref{eq:83}:
    \begin{align}
        \EEp{P^S}{T_S} \le \EEp{P}{T_S} + T\tv(\mathbb{P},\mathbb{P}^S) 
        \le \EEp{P}{T_S} + T\sqrt{\frac{1}{2}\kl(\mathbb{P}\|\mathbb{P}^S)}
        \le \EEp{P}{T_S} + 2\varepsilon T\sqrt{\frac{2^{-m}}{3}\EEp{P}{T_S}}. \label{eq:pinsker3}
    \end{align}
    Since exactly one action is selected at each round, $\sum_{S\subseteq[K], |S|=m}\EEp{P}{T_S} = T$. Averaging Equation~\eqref{eq:gap} over all valid actions, by Equation~\eqref{eq:pinsker3} and Jensen's inequality:
    \begin{align*}
        \frac{1}{M}\sum_{S\subseteq[K], |S|=m}\EEp{P^S}{R_T} 
        &\ge 2^{-m}\varepsilon\left(T-\frac{1}{M}\sum_{S\subseteq[K], |S|=m}\EEp{P^S}{T_S}\right) \\
        &\ge 2^{-m}\varepsilon\left(T-\frac{1}{M}\sum_{S\subseteq[K], |S|=m} \EEp{P}{T_S} - \frac{2\varepsilon T}{M}\sum_{S\subseteq[K], |S|=m}\sqrt{\frac{2^{-m}}{3}\EEp{P}{T_S}}\right) \\
        &\ge 2^{-m}\varepsilon\left(T-\frac{1}{M}\sum_{S\subseteq[K], |S|=m} \EEp{P}{T_S} - 2\varepsilon T\sqrt{\frac{2^{-m}}{3M}\sum_{S\subseteq[K], |S|=m}\EEp{P}{T_S}}\right) \\
        &= 2^{-m}\varepsilon T\left(1-\frac{1}{M} - 2\varepsilon \sqrt{\frac{2^{-m}T}{3M}}\right).
    \end{align*}
    Let
    \begin{equation*}
        \varepsilon = \frac{1}{6}\sqrt{\frac{3M}{2^{-m}T}}.
    \end{equation*}
    Since $T\ge 2^m M /3$, $\varepsilon\le 1/2$ as required.
    Since $K>m\ge 2$, $M={K\choose m}\ge{3\choose 2} = 3$. So:
    \begin{equation*}
        1-\frac{1}{M} - 2\varepsilon \sqrt{\frac{2^{-m}T}{3M}} \ge 1-\frac{1}{3} -\frac{1}{3} = \frac{1}{3},
    \end{equation*}
    and
    \begin{equation*}
        \frac{1}{M}\sum_{S\subseteq[K], |S|=m}\EEp{P^S}{R_T} 
        \ge \frac{\sqrt{2^{-m}3MT}}{18}.
    \end{equation*}
    At least one model $P^S$ induces this much regret. Since each $P^S$ is a distribution over oblivious reward sequences, averaging yields a deterministic reward sequence with at least this much regret.
\end{proof}

\section{Auxiliary Lemmas}

\begin{lemma}\label{lem:mixtures}
    Let $\nu_0,\nu_1$ be two probability measures on measurable space $(\Omega,\mathcal{F})$. For $q\in(0,1)$, let $\mu_q=(1-q)\nu_0+q\nu_1$. Then, for any $q,q'\in(0,1)$:
    \begin{equation*}
        \kl(\mu_q\|\mu_{q'}) \le \frac{2(q-q')^2}{\min\{q',1-q'\}}\tv(\nu_0,\nu_1).
    \end{equation*}
\end{lemma}
\begin{proof}
    Take as reference measure $\lambda=\nu_0+\nu_1$, so that $\nu_0,\nu_1\ll \lambda$. Let $f_0,f_1$ be the densities of $\nu_0,\nu_1$ with respect to $\lambda$. Then, for $q\in(0,1)$, $\mu_q$ has density $g_q=(1-q)f_0+qf_1$. Upper-bounding the KL divergence by the $\chi^2$ divergence and the triangular discrimination by twice the total variation distance (see, e.g., \citealp{Topsoe00}):
    \begin{align}
        \kl(\mu_q\|\mu_{q'}) &
        \le \int_\Omega \frac{(g_q-g_{q'})^2}{g_{q'}} \de \lambda \\
        &= (q'-q)^2\int_\Omega \frac{(f_0-f_1)^2}{(1-q')f_0+q'f_1}\de\lambda \\
        &\le \frac{(q'-q)^2}{\min\{q',1-q'\}}\int_\Omega\frac{(f_0-f_1)^2}{f_0+f_1}\de\lambda \\
        &\le \frac{2(q'-q)^2}{\min\{q',1-q'\}}\tv(\nu_0,\nu_1).
    \end{align}
\end{proof}

\begin{lemma}\label{lem:kluniform}
    Let $S\subseteq [K]$ such that $|S|=m$ and consider two probability models $P$ and $Q$ for a random vector $r\in[0,1]^K$. Let $J\sim\unif(S)$ independently of $r$ and let $Y=(J,r(J))$. Let $P_i,Q_i$ denote the marginal laws of $r(i)$ under $P$ and $Q$, respectively. Then:
    \begin{equation*}
        \kl(Q_Y\|P_Y) = \frac{1}{m}\sum_{i\in S}\kl(Q_i\|P_i).
    \end{equation*}
\end{lemma}
\begin{proof}
    By the chain rule of the KL divergence:
    \begin{equation}
        \kl(Q_Y\|P_Y) = \kl(Q_J\|P_J) + \EEp{J\sim Q_J}{\kl(Q_{r(J)|J}\|P_{r(J)|J})} = \frac{1}{m}\sum_{i\in S}\kl(Q_i\|P_i),
    \end{equation}
    since $P_J=Q_J=\unif(S)$ and $J$ is independent of $r$.
\end{proof}

\section{Experiments: Further Results and Details}\label{app:exp}
We test each algorithm with its feedback model, as in the main text, but on different kinds of reward sequences. In all the experiments from this section, $K=10$, $m=4$, $T=50000$, $G$ is a random set of $m$ good arms, and $\varepsilon=0.1$.

\paragraph{Computing infrastructure.} All experiments were run on a Lenovo ThinkPad X1 2-in-1 Gen 9 with Intel® Core™ Ultra 7 155U × 14 CPU and 16GiB of RAM.

\paragraph{Hyperparameters. } We use the hyperparameters prescribed by the theory. For \algname, we use $\eta=\min\left(\sqrt{\frac{\ln(K/m)}{2KT}},\frac{3-e}{K}\right)$ and $\epsilon=K\eta$ (Theorem~\ref{th:upper}). For OSMA, we use $\eta=\sqrt{\frac{2m\ln(Km)}{KT}}$~\citep[][Theorem 3]{audibert2014regret}. For EXP2, we use $\gamma=\min\left(0.5,\sqrt{\frac{K\ln {K\choose m}}{3T}}\right)$ and $\eta=\gamma/(mK)$~\citep[][Theorem 4]{bubeck2012towards}. For $K$-Metaplayer, we use $\eta=\sqrt{\frac{\ln {K\choose m}}{mKT}}$~\citep[][Lemma 1]{AlaturL020}.

\paragraph{Stationary stochastic correlated rewards (100 realizations).} This is the reward family described in Section~\ref{sec:exp}. We just repeat the experiment of Figure~\ref{fig:correlated} on 100 independently generated instances. The results are shown in Figure~\ref{fig:full}. In all other experiments, the instance is held fixed across independent runs and the seed just controls the randomness of the algorithm and of the feedback.

\begin{figure}[t]
    \centering
    \includegraphics[width=\linewidth]{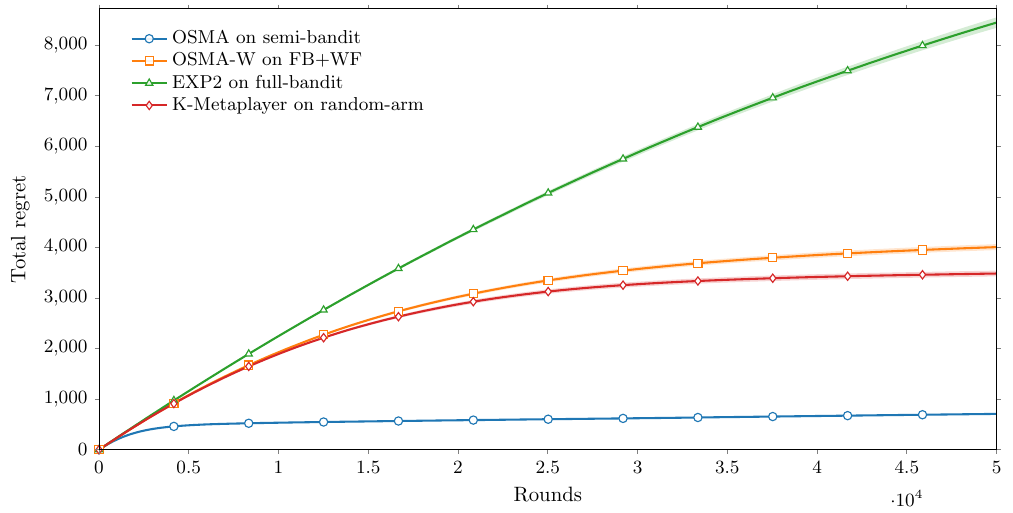}
    \caption{Results on stationary correlated rewards. Total regret per round, averaged over 100 independent reward sequences, with 99\% bootstrap confidence intervals.}
    \label{fig:full}
\end{figure}

\paragraph{Stationary uniform rewards.} The rewards are independent and uniform, $r_{t,i} = U_{t,i} + \varepsilon\indi{i\in G}$, where $U_{t,i}\sim\unif([0,1-\varepsilon])$. The results are shown in Figure~\ref{fig:uniform}

\begin{figure}[t]
    \centering
    \includegraphics[width=\linewidth]{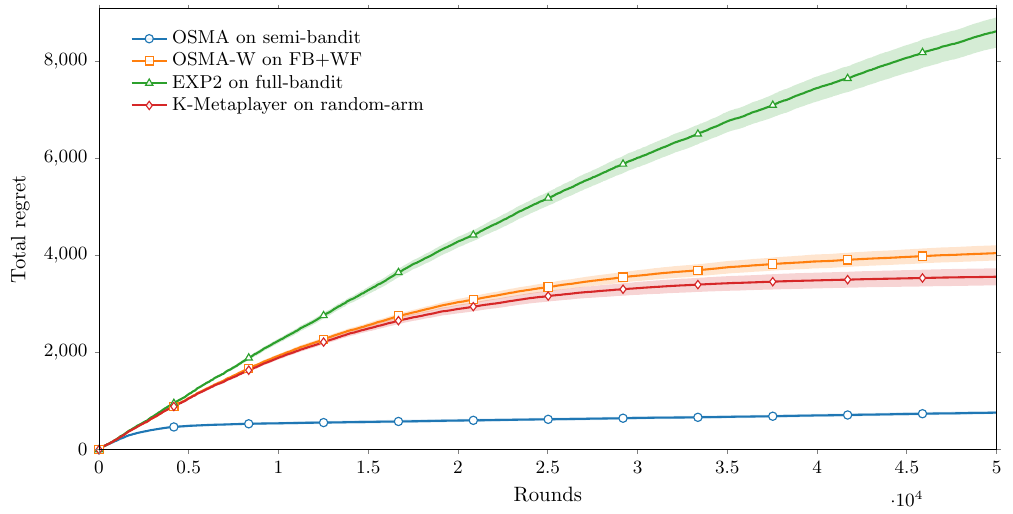}
    \caption{Results on stationary uniform rewards. Total regret per round, averaged over 10 seeds, with 95\% bootstrap confidence intervals.}
    \label{fig:uniform}
\end{figure}



\paragraph{Nonstationary uniform rewards.} Rewards are again defined as $r_{t,i} = U_{t,i} + \varepsilon\indi{i\in G_t}$, but the set of good arms is time-dependent. In particular, we consider $4$ blocks of geometrically increasing size. The set of good arms is sampled independently for each block and held fixed for the duration of the block. In Figure~\ref{fig:blocks} we plot the regret with respect to the best per-block $m$-subset. This is a stronger notion than what considered in our theory, but all algorithms seem capable of keeping this dynamic regret sublinear.

\begin{figure}[t]
    \centering
    \includegraphics[width=\linewidth]{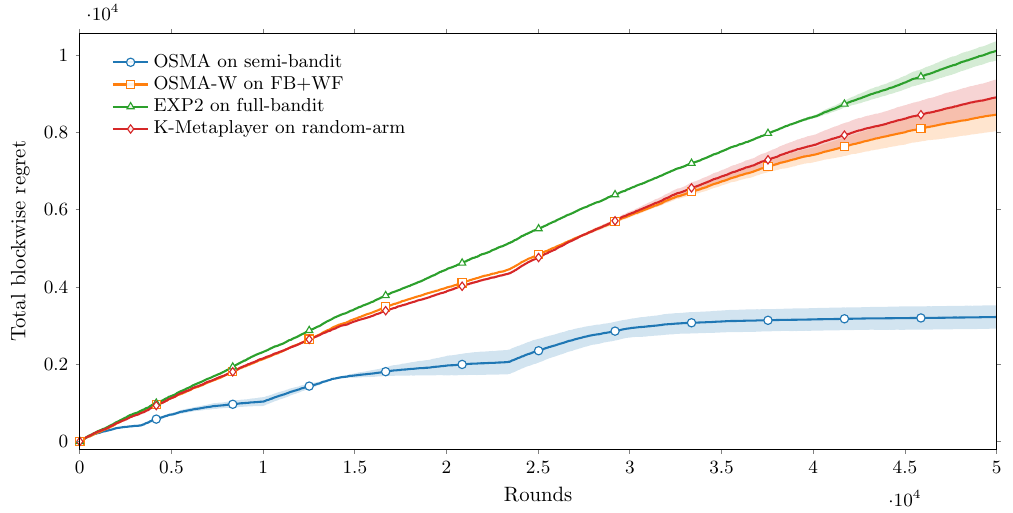}
    \caption{Results on nonstationary rewards. Total dynamic regret per round, averaged over 10 seeds, with 95\% bootstrap confidence intervals.}
    \label{fig:blocks}
\end{figure}

\paragraph{Stationary Bernoulli rewards.} The rewards are independent and binary, $r_{t,i}\sim \ber(U_{t,i}+\varepsilon\indi{i\in G})$, $U_{t,i}\sim\unif([0,1-\varepsilon])$. The results are shown in Figure~\ref{fig:bernoulli}.

\begin{figure}[t]
    \centering
    \includegraphics[width=\linewidth]{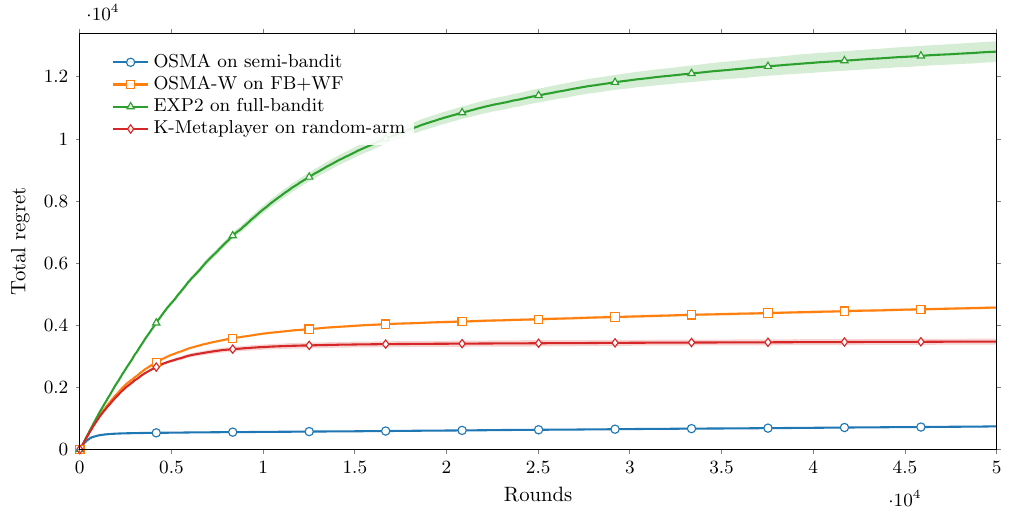}
    \caption{Results on binary rewards. Total regret per round, averaged over 10 seeds, with 95\% bootstrap confidence intervals.}
    \label{fig:bernoulli}
\end{figure}

\end{document}